\documentclass{article}

\usepackage{microtype}
\usepackage{graphicx}
\usepackage{subfigure}
\usepackage{booktabs} 

\usepackage{hyperref}

\usepackage[accepted]{icml2023}

\usepackage{amsmath}
\usepackage{amssymb}
\usepackage{mathtools}
\usepackage{amsthm}
\usepackage{thmtools}
\usepackage{bbm}
\usepackage{thm-restate}

\usepackage[capitalize,noabbrev]{cleveref}
\crefrangelabelformat{section}{#3#1#4--#5\crefstripprefix{#1}{#2}#6}

\theoremstyle{plain}

\theoremstyle{definition}

\theoremstyle{remark}

\newcommand\numberthis{\addtocounter{equation}{1}\tag{\theequation}}
\renewcommand{\d}[1]{\ensuremath{\operatorname{d}\!{#1}}}
\DeclareMathOperator{\Var}{Var}

\declaretheorem[name=Theorem,numberwithin=section]{thm}

\declaretheorem[name=Lemma,numberwithin=section]{lem}

\declaretheorem[name=Assumption]{asm}

\usepackage[textsize=tiny]{todonotes}

\icmltitlerunning{Non-Normal Diffusion Models}

\begin{document}

\twocolumn[
\icmltitle{Non-Normal Diffusion Models}




\begin{icmlauthorlist}
\icmlauthor{Henry Li}{yale}
\end{icmlauthorlist}

\icmlaffiliation{yale}{Department of Computer Science and Applied Mathematics, Yale University, New Haven, CT}

\icmlcorrespondingauthor{Henry Li}{henry.li@yale.edu}

\icmlkeywords{diffusion models, generative modeling, likelihood modeling, density estimation}

\vskip 0.3in
]




\begin{abstract}
Diffusion models generate samples by incrementally reversing a process that turns data into noise. We show that when the step size goes to zero, the reversed process is invariant to the distribution of these increments. This reveals a previously unconsidered parameter in the design of diffusion models: the distribution of the diffusion step $\boldsymbol\Delta \mathbf{x}_k := \mathbf{x}_{k} - \mathbf{x}_{k + 1}$. This parameter is implicitly set by default to be normally distributed in most diffusion models. By lifting this assumption, we generalize the framework for designing diffusion models and establish an expanded class of diffusion processes with greater flexibility in the choice of loss function used during training. We demonstrate the effectiveness of these models on density estimation and generative modeling tasks on standard image datasets, and show that different choices of the distribution of $\boldsymbol\Delta \mathbf{x}_k$ result in qualitatively different generated samples.
\end{abstract}

\section{Introduction}
Diffusion models \cite{sohl2015deep,ho2020denoising,song2020score,vahdat2020nvae,dhariwal2021diffusion} have quickly established themselves as one of the most powerful classes of generative models in an already crowded and competitive space --- one which also includes GANs \cite{goodfellow2020generative,brock2018large,karras2019style}, VAEs \cite{kingma2013auto,vahdat2020nvae,child2020very}, flows \cite{dinh2014nice,kingma2018glow,dinh2016density}, and autoregressive models \cite{salimans2017pixelcnn++,oord2016wavenet,child2019generating}, among others. 


A standard assumption for diffusion models is that $\boldsymbol\Delta_{\mathbf{x}_k} := \mathbf{x}_{k} - \mathbf{x}_{k + 1}$ are normally distributed \cite{sohl2015deep,ho2020denoising,song2020score,ho2022cascaded}.  However, there are many known cases in physical and biological systems where the random incremental behavior of particles colliding in a space does not follow the standard Gaussian distribution \cite{hidalgo2020hitchhiker,cugliandolo2002dynamics}. These examples are also called anomalous diffusions \cite{gefen1983anomalous,bouchaud1990anomalous}. In this work, we consider such a scenario, and propose a generalized framework for modeling diffusion models with minimal assumptions on the distribution of the $\boldsymbol\Delta_{\mathbf{x}_k}$. To develop this framework, we prove a novel result on the convergence of non-time homogeneous random walks to stochastic processes in the limit of small time steps. Finally, we demonstrate that our  framework allows for greater freedom in the design of the model and its training dynamics, while retaining competitive generative modeling capabilities in terms of both model likelihood and sample quality.

\section{Background}


Diffusion models \cite{sohl2015deep,ho2020denoising,song2020score} take the form $p_\theta(\mathbf{x}) = \int p_\theta(\mathbf{x}_{0:T}) \d{\mathbf{x}}_{1:T}$ where data $\mathbf{x}_0 := \mathbf{x}$ are related to a set of latent variables $\mathbf{x}_{1:T} := (\mathbf{x}(t_{1}), \dots, \mathbf{x}(t_T))$ distributed as marginals of a diffusion process governed by an It\^{o} stochastic differential equation (SDE)
\begin{equation}
    \d{\mathbf{x}} = \mathbf{f}(\mathbf{x}, t)\d{t} + g(t)\d{\mathbf{w}}
    \label{eq:ito_sde}
\end{equation}
with respect to time points $\{t_k\}_{k=1}^T$. $\mathbf{f}$ and $g$ are typically called \textit{drift} and \textit{diffusion} functions, and $\mathbf{w}$ is the standard Wiener process. Samples can then be generated by modeling the reverse diffusion, which has a simple form given by \cite{anderson1982reverse}
\begin{equation}
    \d{\mathbf{x}} = [\mathbf{f}(\mathbf{x}, t) - g(t)^2 \underbrace{\nabla_\mathbf{x} \log p(\mathbf{x}, t)}_{\approx \mathbf{s}_\theta(\mathbf{x}, t)}]\d{t} + g(t)\d{\overline{\mathbf{w}}},
    \label{eq:reverse_sde}
\end{equation}
where $\overline{\mathbf{w}}$ is a reverse-time Wiener process. Note that Eq. (\ref{eq:reverse_sde}) is itself an It\^{o} SDE of the form Eq. (\ref{eq:ito_sde}). Training the diffusion model involves approximating the true score function $\nabla_\mathbf{x} \log p(\mathbf{x}, t)$ with a neural network $\mathbf{s}_\theta(\mathbf{x}, t)$ in Eq. (\ref{eq:reverse_sde}). This can be achieved directly via score matching \cite{hyvarinen2005estimation,song2019generative,song2020score}, or by modeling the sampling process \cite{sohl2015deep,ho2020denoising,kingma2021variational}, which is obtained by discretizing the reverse-time SDE into a Markov chain with joint likelihood
\begin{equation}
    p_\theta(\mathbf{x}_{0:T}) = p(\mathbf{x}_T) \prod_{k=0}^{T-1} \nu_\theta(\mathbf{x}_k | \mathbf{x}_{k + 1})
    \label{eq:markov_approximation2}
\end{equation}
or equivalently
\begin{equation}
    p_\theta(\mathbf{x}_{0:T}) = p(\mathbf{x}_T) \prod_{k=0}^{T-1} \rho_\theta(\boldsymbol\Delta_{\mathbf{x}_k} | \mathbf{x}_{k + 1}),
    \label{eq:markov_approximation}
\end{equation}
where $\nu_\theta, \rho_\theta$ are Markov models and $\boldsymbol\Delta_{\mathbf{x}_k} := \mathbf{x}_{k+1} - \mathbf{x}_k$. While most works e.g. \cite{song2020score,ho2020denoising,kingma2021variational} model Eq. \eqref{eq:markov_approximation2}, we shall turn our attention to the equivalent formulation Eq. \eqref{eq:markov_approximation}, which focuses on the \textit{increments}, rather than the \textit{marginals} of the diffusion. Letting $q$ be the density of the Gaussian process Eq. \ref{eq:reverse_sde}, Eqs. (\ref{eq:markov_approximation2}) and  (\ref{eq:markov_approximation}) result in the same likelihood bound
\begin{align*}
    \log p_\theta(\mathbf{x}) &\geq \mathbb{E}_q\bigg[\underbrace{\log p(\mathbf{x}_0|\mathbf{x}_1)}_{\mathcal{L}_0} \\
    &- \sum_{k=1}^T \underbrace{KL(q(\boldsymbol\Delta_{\mathbf{x}_k} | \mathbf{x}_{k + 1}) || p_\theta(\boldsymbol\Delta_{\mathbf{x}_k} | \mathbf{x}_{k + 1}))}_{\mathcal{L}_k} \\
    &- \underbrace{KL(q(\mathbf{x}_T) || p(\mathbf{x}_T))}_{\mathcal{L}_T} \bigg]
    \numberthis
    \label{eq:ll_bound}
\end{align*}
that reduces to a simple function of $\mathbf{s}_\theta(\mathbf{x}, t)$.

When forming approximations such as Eq. (\ref{eq:markov_approximation}), it is important to consider the conditions under which they converge to Eq. (\ref{eq:reverse_sde}). While this convergence is known for normally distributed $\boldsymbol\Delta_{\mathbf{x}_k}$ \cite{sohl2015deep,song2020score,sarkka2019applied}, we shall extend this result to arbitrarily distributed $\boldsymbol\Delta_{\mathbf{x}_k}$ in Section \ref{sec:random_walk_limit}. 


Ultimately, either choice of learning $\mathbf{s}_\theta(\mathbf{x}, t)$ allows for unbiased estimates of $\log p_\theta(\mathbf{x})$ by modeling the probability flow ODE (PF-ODE) corresponding to Eq. (\ref{eq:reverse_sde}), which can be derived via the Fokker-Planck equation \cite{song2020score}
\begin{equation}
    \d{\mathbf{x}} = \left[\mathbf{f}(\mathbf{x}, t) - \frac{1}{2}g(t)^2 \nabla_\mathbf{x} \log p(\mathbf{x}, t)\right]\d{t},
\end{equation}
and substituting the score with $\mathbf{s}_\theta(\mathbf{x}, t)$.

\begin{table*}
  \centering
  \resizebox{\textwidth}{!}{
  \begin{tabular}{cccc}
  \toprule
  $p$ & $q$ & $KL(p(\mathbf{x}) || q(\mathbf{x}))$ & $\mathcal{L}_k$ (note: **) \\
  \midrule
    $\mathcal{N}(\boldsymbol\mu_1, \sigma^2)$ & $\mathcal{N}(\boldsymbol\mu_2, \sigma^2)$ & $\frac{1}{2 \sigma^2} ||\boldsymbol\mu_1 - \boldsymbol\mu_2||^2$ & $\mathbb{E} \left[ w_k ||\mathbf{r}_k||^2\right]$ \\
    $\text{Laplace}(\boldsymbol\mu_1, \sigma^2)$ & $\text{Laplace}(\boldsymbol\mu_2, \sigma^2)$ & $\exp\left(-\frac{|\mu_2 - \mu_1|}{\sigma}\right) + \frac{|\mu_2 - \mu_1|}{\sigma} + 1$ & $\mathbb{E} 
\big[\exp(-v_k||\mathbf{r}_k||_1) - 1 + v_k||\mathbf{r}_k||_1\big]$ \\
    $\text{Uniform}[\boldsymbol\mu_1 - \sqrt{3}\sigma, \boldsymbol\mu_1 + \sqrt{3}\sigma]$ & $\mathcal{N}(\boldsymbol\mu_2, \sigma^2)$ & $\frac{1}{2}\left(\frac{1}{\sigma^2}  (\mu_1 - \mu_2)^2 + \log \frac{\pi}{6} + 1 \right)$ & $w_k \mathbb{E}_{\boldsymbol\epsilon} ||\mathbf{r}_k||^2 + \frac{1}{2}(1 + \log \sqrt{\frac{\pi}{6}})$ \\
    $\text{Uniform}[\boldsymbol\mu_1 - \sqrt{3}\sigma, \boldsymbol\mu_1 + \sqrt{3}\sigma]$ & $\text{Laplace}(\boldsymbol\mu_2, \sigma^2)$ & $\begin{cases} \frac{1}{2\sigma^2} (\mu_1 - \mu_2)^2 + \frac{1}{2} & \mu_2 \in A^* \\ \frac{1}{\sigma}|\mu_1 - \mu_2| & \mu_2 \notin A\end{cases}$ & $\begin{cases}
w_k \mathbb{E} ||\mathbf{r}_k||^2_2 + \frac{1}{2} & \text{if} \hspace{.1in} \boldsymbol\epsilon_\theta(\mathbf{x}, t) \in A \\
v_k \mathbb{E}||\mathbf{r}_k||_1 & \text{if} \hspace{.1in} \boldsymbol\epsilon_\theta(\mathbf{x}, t) \notin A \\
\end{cases}$ \\
\bottomrule
  \end{tabular}
  }
  \caption{Summary of the diffusion models proposed in Section \ref{sec:gdpm}. *$A = [\mu_1 - b_1, \mu_1 + b_1]$. **$\mathbf{r}_k := \boldsymbol\epsilon - \boldsymbol\epsilon_\theta(\mathbf{x}_k, t_k)$.}
  \label{tab:1}
\end{table*}
\section{Convergence of Non-Normal Random Walks to Diffusion Processes}
\label{sec:random_walk_limit}
A fundamental challenge in diffusion modeling is forming tractable approximations to Eq. (\ref{eq:ito_sde}). Our result is inspired by Donsker's classic Invariance Principle \cite{billingsley2013convergence}, which gives the functional convergence of an unbiased random walk to a standard Brownian motion. We now consider a time-inhomogeneous, biased random walk $\mathbf{x}_k$. Let $\mathbf{x}(t)$ be the solution to Eq. ~\eqref{eq:ito_sde}. Intuitively, one might expect a similar convergence of $\mathbf{x}_k$ to $\mathbf{x}(t)$ if we constrain the first and second moments of its increments $\boldsymbol\Delta_{\mathbf{x}_k} := \mathbf{x}_{k+1} - \mathbf{x}_k$ to be
\begin{equation}
\begin{split}
\mathbb{E}[\boldsymbol\Delta_{\mathbf{x}_k} | \mathbf{x}_k] = \mathbf{f}(\mathbf{x}_k, t_k) \Delta_{t_k} \\
\Var(\boldsymbol\Delta_{\mathbf{x}_k} | \mathbf{x}_k) = g(t_k)^2 \Delta_{t_k} \label{eq:markov_1_2_moments}
\end{split}.
\end{equation}

This type of convergence has been previously explored for normally distributed $\boldsymbol\Delta_{\mathbf{x}_k}$ in diffusion modeling \cite{sohl2015deep,ho2020denoising,song2020score}, and is well known in general SDE literature \cite{sarkka2019applied,oksendal2003stochastic,kloeden1992stochastic}. More generalized results also exist for time-homogeneous or equilibrium state processes \cite{ethier2009markov,vidov2009analytical,stroock2013introduction}. However, there does not exist to our knowledge a convergence result for the case of general $\boldsymbol\Delta_{\mathbf{x}_k}$ in our non-equilibrium case \cite{sohl2015deep}. Here we shall provide such a result, and show that convergence occurs with surprisingly few assumptions. This inspires a generalized framework for designing diffusion probabilistic models where the distribution of $\boldsymbol\Delta_{\mathbf{x}_k}$ is left as a tunable free parameter. We leverage this framework in Section \ref{sec:gdpm} to define a generalized class of diffusion probabilistic models.

\subsection{Structured Random Walks}
\label{sec:srw}
Let $\mathbf{x}_k$ be a random walk. We introduce the following notion of structure, which allows us to characterize a random walk entirely in terms of the drift and diffusion functions $\mathbf{f}$ and $g$, the time step $\Delta_{t_k}$, and a sequence of independent variables $\mathbf{z}_k$.

\begin{restatable}[Structured Random Walks]{defn}{srw}
\label{defn:srw}
We say that a random walk $\mathbf{x}_k$ is \textbf{structured} (with respect to an It\^{o} SDE) when its increments $\boldsymbol\Delta_{\mathbf{x}_k} := \mathbf{x}_{k + 1} - \mathbf{x}_{k}$ support the decomposition
\begin{equation}
    \boldsymbol\Delta_{\mathbf{x}_k} = \mathbf{f}(\mathbf{x}_k, t_k)\Delta_{t_k} + g(t_k) \sqrt{\Delta_{t_k}} \mathbf{z}_k,
    \label{eq:structured_random_walk}
\end{equation}
where $\mathbb{E}[\mathbf{z}_k] = 0$, $\Var(\mathbf{z}_k) = 1$, $\Delta_{t_k} := t_{k+1} - t_k$, and $\mathbf{f}$, $g$ correspond to the drift and diffusion terms of the respective It\^{o} SDE.

\end{restatable}



The structural property in Definition \ref{defn:srw} is quite natural. In fact, it is how diffusion steps are usually computed, e.g., via the reparameterization trick \cite{kingma2013auto,ho2020denoising} or SDE solvers such as the Euler-Maruyama method \cite{song2020score}. Moreover, it satisfies Eq. (\ref{eq:markov_1_2_moments}). If we additionally assume that $\mathbf{f}(\mathbf{x}, t)$ is linear in $\mathbf{x}$, as is the case with the forward diffusion process in standard diffusion models \cite{sohl2015deep,ho2020denoising,song2020score,kingma2021variational}, we have the following closed form representations of its first and second moments at all $k \in \{0, \dots, T\}$.

\begin{restatable}[Moments of Structured Random Walks]{thm}{srwmoments}
\label{thm:random_walk_mu_sigma}
Let $\mathbf{x}_k$ be a structured random walk and $\mathbf{f}(\mathbf{x}, t_k) = \beta(t_k) \mathbf{x}$ be linear. Then
\begin{equation*}
    \boldsymbol\mu(t_k) := \mathbb{E} [\mathbf{x}_k] = \bar{\alpha}_k \mathbf{x}_0 \hspace{.1in} \text{and} \hspace{.1in} \boldsymbol\sigma(t_k)^2 := \Var(\mathbf{x}_k) = \bar{\gamma}_k,
\end{equation*}
where $\bar{\alpha}_k = \prod_{i=1}^k \left(1 + \beta_i \right)$ and $\bar{\gamma}_k = \sum_{i=1}^k \left(\frac{\bar{\alpha}_k}{\bar{\alpha}_{i+1}} g_i\right)^2$. For notational convenience, we let $\beta_i := \beta(t_i)\Delta_{t_k}$ and $g_i := g(t_i) \sqrt{\Delta_{t_k}}$.
\end{restatable}


In diffusion modeling, we are not just interested in computing the moments of $\mathbf{x}_k$ --- we would like to sample from $p(\mathbf{x}_k)$\footnote{Where $p = q$ or $p = p_\theta$.}. This is a difficult task for generally distributed $\boldsymbol\Delta_{\mathbf{x}_k}$, since the distribution of $\mathbf{x}_k = \mathbf{x}_0 + \sum_{i=1}^k \boldsymbol\Delta_{\mathbf{x}_i}$ is usually intractable. To sidestep this issue, many works assume that $\boldsymbol\Delta_{\mathbf{x}_k}$ are normally distributed; since Gaussian random variables are closed under summation and specified by their first and second moments, we see below that Lemma \ref{thm:random_walk_mu_sigma} is sufficient for identifying the distribution of $\mathbf{x}_k$.

\begin{restatable}{cor}{normaldiffusions}
\label{thm:normal_diffusions}
Let $\mathbf{x}_k$ be a structured random walk, $\mathbf{z}_k$ be normally distributed, and $\mathbf{f}(\mathbf{x}, t) = \beta(t) \mathbf{x}$ where $\beta(t)$ is one of the noise schedules defined in \cite{sohl2015deep,ho2020denoising,kingma2021variational}. Then we recover their respective forward processes.
\end{restatable}


\subsection{An Invariance Principle}
Lifting the assumption of normally distributed increments $\boldsymbol\Delta_{\mathbf{x}_k}$, we show that we still ultimately obtain a Gaussian process in the limit as $\Delta_{t_k} \rightarrow 0$. Much like the aforementioned Donsker's theorem, this also gives rise to an invariance --- in the distribution of $\boldsymbol\Delta_{\mathbf{x}_k}$. We once again leverage the notion of \textit{structured} random walks to present a general theorem for the convergence of Markov chains with increments of the form Eq. (\ref{eq:structured_random_walk}).


\begin{restatable}[Structured Invariance Principle]{thm}{invariance}
\label{thm:main}
Suppose regularity conditions hold and $\{\mathbf{x}_k\}_{k=1}^n$ is a structured random walk on $\mathbb{R}^d$. Let $\overline{\mathbf{x}}_T(t) = \mathbf{x}_0 + \sum_{k=1}^{n_t} \boldsymbol\Delta_{\mathbf{x}_k}$ be the continuous-time c\`{a}dl\`{a}g extension of $\mathbf{x}_k$, where $n_t = \lfloor t * T \rfloor$. Then $\overline{\mathbf{x}}_T$ converges in distribution to $\mathbf{x}(t)$, as $\Delta_{t_k}\rightarrow 0$.
\end{restatable}

Theorem \ref{thm:main} outlines the existence of a much larger class of increments $\boldsymbol\Delta_{\mathbf{x}_k}$ that converge to our desired limiting distribution $\mathbf{x}(t)$. The convergence to $\mathbf{x}(t)$ unlocks many of the essential properties for the tractability of diffusion models which we take for granted in Gaussian increments, such as fast sampling from the forward process and a closed form Eq. \eqref{eq:ll_bound}, without the need to assume Gaussian increments. Finally, we verify that we can recover Donsker's theorem when we let $\mathbf{f} = \mathbf{0}$ and $g = 1$.


\section{Non-Normal Diffusion Models}
\label{sec:gdpm}
Leveraging the framework established in Section \ref{sec:random_walk_limit}, we introduce an expanded class of probabilistic diffusion models, centered around alternative distributional assumptions for $q(\boldsymbol\Delta_{\mathbf{x}_k} | \mathbf{x}_{k+1})$ and $p_\theta(\boldsymbol\Delta_{\mathbf{x}_k} | \mathbf{x}_{k+1})$. While the space of viable diffusion models allowed by Theorem \ref{thm:main} effectively contains all distributions of $\boldsymbol\Delta_{\mathbf{x}_k}$ with finite mean and variance, we restrict our study to the following examples and leave further exploration to future work. Detailed derivations can be found in Appendix \ref{sec:ll_bound_derivations}. A summary of all models can be found in Table \ref{tab:1}.

\subsection{Gaussian $q$ and $p_\theta$}
\label{sec:gaussian_dm}
First, we recover the default diffusion model loss term $\mathcal{L}_k$ (from Eq. \ref{eq:ll_bound}) by making the standard assumption that $\boldsymbol\Delta_{\mathbf{x}_k}$ are normally distributed. Since the space of Gaussian-distributed random variables is closed under affine operations, we trivially obtain the convergence of the random walk (Eq. \ref{eq:markov_approximation}) to a Gaussian process. Using the closed form mean and variance terms of a linear ODE \cite{sarkka2019applied}, we obtain
\begin{equation}
\mathcal{L}_k = w_k \mathbb{E}_{\boldsymbol\epsilon} ||\boldsymbol\epsilon - \boldsymbol\epsilon_\theta(\mathbf{x}_k, t_k)||^2,
\label{eq:l2_ll_bound}
\end{equation}
where $w_k = \frac{g(t_k)^2}{2\sigma(t_k)^2} \Delta_{t_k}$ and $\boldsymbol\epsilon_\theta(\mathbf{x}_k, t_k) = \sigma(t_k) \mathbf{s}_\theta(\mathbf{x}_k, t_k)$. Plugging Eq. \ref{eq:l2_ll_bound} into the likelihood bound Eq. \ref{eq:ll_bound}, we see that maximizing the likelihood of a standard diffusion model with Gaussian increments minimizes a quadratic error term between the score function $\mathbf{s}_\theta(\mathbf{x}_k, t_k) = \frac{1}{\sigma(t_k)}\boldsymbol\epsilon_\theta(\mathbf{x}_k, t_k)$.

\subsection{Laplace $q$ and $p_\theta$}
\label{sec:laplace_dm}
We now consider the case of Laplace distributed $\boldsymbol\Delta_{\mathbf{x}_k}$. Invoking Theorem \ref{thm:main}, we can derive the alternative loss
\begin{equation}
\begin{split}
\mathcal{L}_k = \mathbb{E}_{\boldsymbol\epsilon} 
\big[&\exp(-v_k||\boldsymbol\epsilon - \boldsymbol\epsilon_\theta(\mathbf{x}_k, t_k)||_1) - 1 \\
&+ v_k||\boldsymbol\epsilon - \boldsymbol\epsilon_\theta(\mathbf{x}_k, t_k)||_1\big],
\label{eq:l1_ll_bound}
\end{split}
\end{equation}
where $v_k := \sqrt{w_k}$.

While the term in the expectation $\mathbf{d}(\boldsymbol\epsilon - \boldsymbol\epsilon_\theta(\mathbf{x}_k, t_k)) := \exp\left(-v_k||\boldsymbol\epsilon - \boldsymbol\epsilon_\theta(\mathbf{x}_k, t_k)||_1\right) - 1 + v_k||\boldsymbol\epsilon - \boldsymbol\epsilon_\theta(\mathbf{x}_k, t_k)||_1$ appears somewhat opaque, we can see that it converges to a weighted $L1$ norm of the error $\mathbf{r}_k := \boldsymbol\epsilon - \boldsymbol\epsilon_\theta(\mathbf{x}_k, t_k)$ under two conditions:
\begin{equation}
    \lim_{t_k \rightarrow 0} \frac{v_k ||\boldsymbol\epsilon - \boldsymbol\epsilon_\theta(\mathbf{x}_k, t_k)||_1}{\mathbf{d}(\boldsymbol\epsilon - \boldsymbol\epsilon_\theta(\mathbf{x}_k, t_k))} = 1,
\end{equation}
i.e., when $t$ is small, and
\begin{equation}
\lim_{||\mathbf{r}_k||_1 \rightarrow \infty} \frac{v_k ||\boldsymbol\epsilon - \boldsymbol\epsilon_\theta(\mathbf{x}_k, t_k)||_1}{\mathbf{d}(\boldsymbol\epsilon - \boldsymbol\epsilon_\theta(\mathbf{x}_k, t_k))} = 1
\end{equation}
i.e., when $||\mathbf{r}_k||$ is large.

\subsection{Uniform $q$, Gaussian $p_\theta$}
\label{sec:gaussian_unif_dm}
Next, we note that $q$ and $p_\theta$ need not be the same family of distributions to apply our framework. To illustrate this, we let $q$ be uniformly distributed on the interval $[\mu_1 - \sqrt{3}\sigma, \mu_2 + \sqrt{3}\sigma$, and $p$ be Gaussian distributed. This results in the familiar form
\begin{equation}
\mathcal{L}_k = w_k \mathbb{E}_{\boldsymbol\epsilon} ||\boldsymbol\epsilon - \boldsymbol\epsilon_\theta(\mathbf{x}_k, t_k)||^2 + C,
\end{equation}
where $C = \frac{1}{2}\left(1 + \log \frac{\pi}{6}\right) \approx 0.34$ may be seen as an additional distributional mismatch penalty incurred by the joint combination of the uniform and normal distributions. We note, however, that such a penalty does not always arise when $p_\theta$ and $q$ are not from the same family of distributions.

\begin{table}
\centering
\begin{tabular}{cccc}
    \toprule
    $q$  & $p_\theta$ & BPD & FID \\
    \midrule
    Gaussian & Gaussian & 2.49 & 1.98 \\
    Laplace & Laplace & 2.47 & 2.44 \\
    Uniform & Gaussian & 2.82 & 1.99 \\
    Uniform & Laplace & 2.66 & 2.39 \\
    \bottomrule
\end{tabular}
\caption{Comparison between the proposed diffusion models on the CIFAR10 dataset. We evaluate in terms of negative log-likelihood (BPD, lower is better) and sample quality (FID, lower is better). BPD and FID are computed with different architectures.}
\label{tab:2}
\end{table}

\subsection{Uniform $q$, Laplace $p_\theta$}
\label{sec:laplace_unif_dm}
Finally, we demonstrate that the phase transition in Section \ref{sec:laplace_dm} to an $L1$-based loss is made explicit in the case where $q$ is uniform and $p_\theta$ is the Laplace distribution. This configuration of distributions produces the piecewise loss
\begin{equation}
\mathcal{L}_k = 
\begin{cases}
w_k \mathbb{E}_{\boldsymbol\epsilon} ||\boldsymbol\epsilon - \boldsymbol\epsilon_\theta(\mathbf{x}_k, t_k)||^2_2 + \frac{1}{2} & \text{if} \hspace{.1in} \boldsymbol\epsilon_\theta(\mathbf{x}, t) \in A \\
v_k \mathbb{E}_{\boldsymbol\epsilon} ||\boldsymbol\epsilon - \boldsymbol\epsilon_\theta(\mathbf{x}_k, t_k)||_1 & \text{if} \hspace{.1in} \boldsymbol\epsilon_\theta(\mathbf{x}, t) \notin A \\
\end{cases},
\end{equation}
where $A = [\mu_1 - \sqrt{3}\sigma w_k, \mu_1 + \sqrt{3}\sigma w_k]$. Now, it is clear that $\mathcal{L}_k$ acts as a linear function in two cases. First, when $t_k \rightarrow 0$, as $A$ becomes a vanishingly small set. And second, when $\mathbf{r}_k := \boldsymbol\epsilon - \boldsymbol\epsilon_\theta(\mathbf{x}_k, t_k)$ is large. Both imply that $\boldsymbol\epsilon_\theta(\mathbf{x}, t) \notin A$.

\section{Experiments}
\label{sec:experiments}
For illustrative purposes, we evaluate the diffusion models proposed in Section \ref{sec:gdpm} on the CIFAR10 \cite{krizhevsky2009learning} and down-sampled ImageNet \cite{van2016pixel} datasets. We quantify the performance of our models with the negative log-likelihood in terms of bits per dimension (BPD) and the Frechet Inception Distance \cite{frechet}. Results are displayed in Table \ref{tab:2}. We show that our model obtains competitive results in terms of both metrics.

More interestingly, some of the losses proposed in Section \ref{sec:gdpm} result in generated samples with distinctly different visual characteristics. For example, images generated by the Laplace-based diffusion models exhibit markedly more saturated colors (Figure \ref{fig:comparison}.

\begin{figure}
    \centering
    \includegraphics[width=0.48\textwidth]{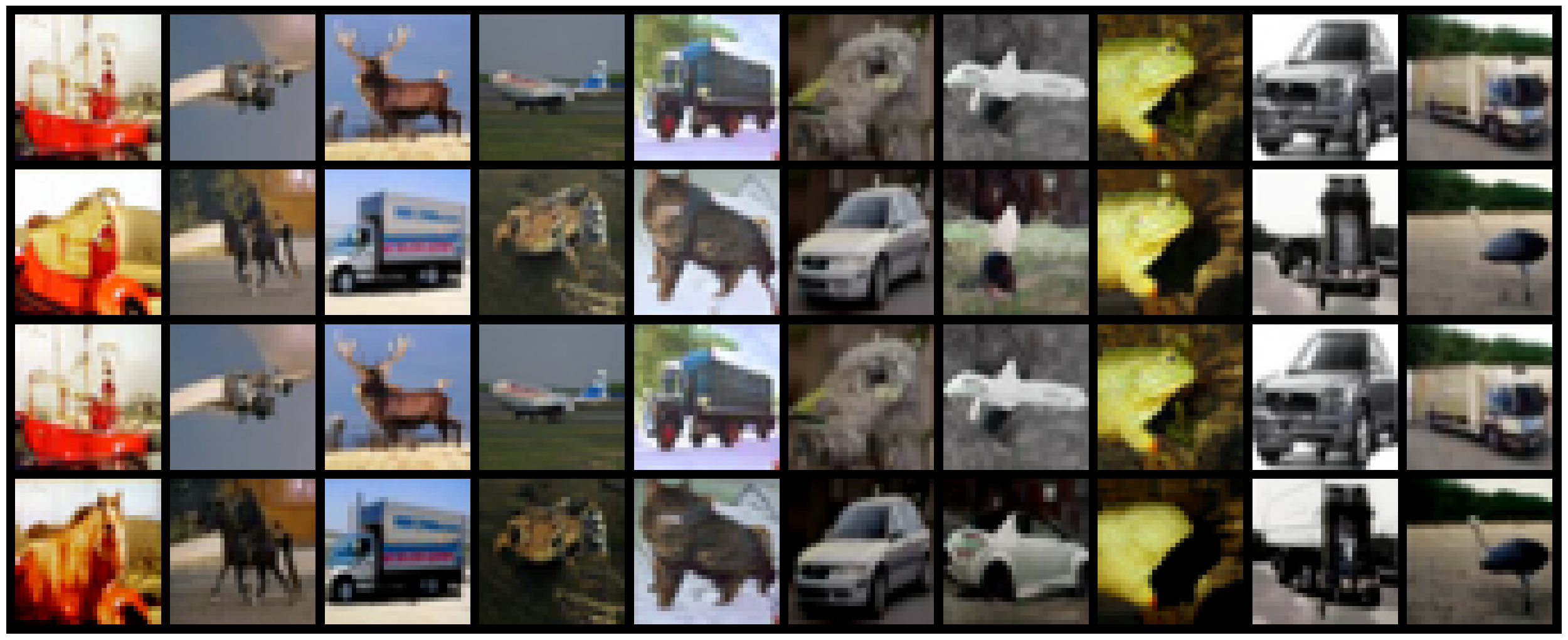}
    \caption{Images generated from the same seed via (in order from top to bottom) Gaussian-Gaussian, Laplace-Laplace, Uniform-Gaussian, and Uniform-Laplace diffusion increments. While the qualitative difference is somewhat subtle, Laplace diffusion appears to be biased towards smoother images with more saturated colors.}
    \label{fig:comparison}
\end{figure}

\section{Conclusion and Limitations}
We derived a probabilistic framework for designing more diverse diffusion models by showing an invariance to the distribution of the diffusion step $\boldsymbol\Delta_{\mathbf{x}_k} := \mathbf{x}_k - \mathbf{x}_{k + 1}$. Freeing up the distributional assumption on $\boldsymbol\Delta_{\mathbf{x}_k}$ allows the end-user greater control over the stylistic qualities of the generative model. An open question is whether score matching under an EMD norm enjoys the same statistical guarantees as the standard score matching objective, e.g., consistency, efficiency, and asymptotic normality \cite{hyvarinen2006consistency,song2020sliced}. We hope that our theoretical framework opens the door for the further diversity and improvements in the design of diffusion models.

\bibliography{main}
\bibliographystyle{icml2023}

\newpage
\appendix
\onecolumn

\section{Derivations}

\subsection{KL Divergence Between Laplace Distributions}
\label{sec:laplace_kl}
For completeness, we provide a derivation for the KL divergence between two Laplace distributions. Let $p$ and $q$ be density functions of distributions $\text{Laplace}(\mu_1, b_1)$ and $\text{Laplace}(\mu_2, b_2)$, i.e.,
\begin{align}
    p(x) &= \frac{1}{2b_1} \exp\left(-\frac{|x - \mu_1|}{b_1}\right) \\
    q(x) &= \frac{1}{2b_2} \exp\left(-\frac{|x - \mu_2|}{b_2}\right).
\end{align}

Then the KL divergence between the two distributions can be written as
\begin{equation}
    KL(p(x)||q(x)) = \underbrace{\int_{-\infty}^\infty p(x) \log p(x) dx}_{*} - \underbrace{\int_{-\infty}^\infty p(x) \log q(x) dx}_{**}
    \label{eq:kl_decomposition}
\end{equation}
We will first approach $**$ as its solution will give us $*$. Plugging in $p$ and $q$, we have
\begin{align*}
    -\int_{-\infty}^\infty p(x) \log q(x) dx 
    &= \int_{-\infty}^\infty \frac{|x - \mu_2|}{2b_1 b_2} \exp\left(-\frac{|x - \mu_1|}{b_1}\right) dx + \log(2b_2),
\end{align*}
where in the case that $\mu_1 > \mu_2$, the integral can be written as
\begin{align*}
    \int_{-\infty}^\infty \frac{|x - \mu_2|}{2b_1 b_2} \exp&\left(-\frac{|x - \mu_1|}{b_1}\right) dx \\
    &= \int_{-\infty}^{\mu_2} \frac{\mu_2 - x}{2b_1 b_2} \exp\left(-\frac{\mu_1 - x}{b_1}\right) dx + \int_{\mu_2}^{\mu_1} \frac{x - \mu_2}{2b_1 b_2} \exp\left(-\frac{\mu_1 - x}{b_1}\right) dx \\
    &\hspace{.3in}+\int_{\mu_1}^\infty \frac{x - \mu_2}{2b_1 b_2} \exp\left(-\frac{x - \mu_1}{b_1}\right) dx \\
    &= \left[\frac{b_1}{2b_2} \exp\left(-\frac{\mu_1 - \mu_2}{b_1}\right)\right] + \left[\frac{\mu_1 - \mu_2 - b_1}{2b_2} + \frac{b_1}{2b_2}\exp\left(-\frac{\mu_1 - \mu_2}{b_1}\right) \right] \\
    &\hspace{.3in}+ \left[ \frac{\mu_1 - \mu_2 + b_1}{2b_2} \right] \\
    &= \frac{\mu_1 - \mu_2}{b_2} + \frac{b_1}{b_2}\exp\left(-\frac{\mu_1 - \mu_2}{b_1}\right),
\end{align*}
and similarly for the case $\mu_1 \leq \mu_2$,
\begin{align*}
    \int_{-\infty}^\infty \frac{|x - \mu_2|}{2b_1 b_2} \exp&\left(-\frac{|x - \mu_1|}{b_1}\right) dx \\
    &= \int_{-\infty}^{\mu_1} \frac{\mu_2 - x}{2b_1 b_2} \exp\left(-\frac{\mu_1 - x}{b_1}\right) dx + \int_{\mu_1}^{\mu_2} \frac{\mu_2 - x}{2b_1 b_2} \exp\left(-\frac{x - \mu_1}{b_1}\right) dx \\
    &\hspace{.3in}+\int_{\mu_2}^\infty \frac{x - \mu_2}{2b_1 b_2} \exp\left(-\frac{x - \mu_1}{b_1}\right) dx \\
    &= \left[ \frac{\mu_2 - \mu_1 + b_1}{2b_2} \right] + \left[\frac{b_1}{2b_2}\exp\left(-\frac{\mu_2 - \mu_1}{b_1}\right) + \frac{\mu_2 - \mu_1 - b_1}{2b_2} \right] \\
    &\hspace{.3in}+ \left[\frac{b_1}{2b_2} \exp\left(-\frac{\mu_2 - \mu_1}{b_1}\right)\right] \\
    &= \frac{\mu_2 - \mu_1}{b_2} + \frac{b_1}{b_2}\exp\left(-\frac{\mu_2 - \mu_1}{b_1}\right).
\end{align*}
Combining both cases and returning to the original cross entropy term, we have
\begin{equation}
    -\int_{-\infty}^\infty p(x) \log q(x) dx 
    = \frac{|\mu_2 - \mu_1|}{b_2} + \frac{b_1}{b_2}\exp\left(-\frac{|\mu_2 - \mu_1|}{b_1}\right) + \log(2b_2).
\end{equation}
Now, letting $p = q$ we can compute the entropy term as
\begin{equation}
    \int_{-\infty}^\infty p(x) \log p(x) dx 
    = - 1 - \log(2b_1).
\end{equation}
Thus, we can conclude that
\begin{equation}
    KL(p(x)||q(x)) = \frac{b_1}{b_2}\exp\left(-\frac{|\mu_2 - \mu_1|}{b_1}\right) + \frac{|\mu_2 - \mu_1|}{b_2} + \log\frac{b_2}{b_1} - 1.
\end{equation}

\subsection{KL Divergence Between a Gaussian Distribution and a Bounded Uniform Distribution}
\label{sec:gaussian_unif_kl}
Let $p$ and $q$ denote the density functions of the $\text{Uniform}([\mu_1 -b_1, \mu_1 + b_1])$ and $\mathcal{N}(\mu_2, \sigma_2)$ distributions, respectively. Then
\begin{align}
    p(x) &= \mathbbm{1}_{x \in [\mu_1 - b_1, \mu_1 + b_1]}\frac{1}{2b_1} \\
    q(x) &= \frac{1}{\sigma \sqrt{2 \pi} } \exp\left(-\left(\frac{x - \mu_2}{b_2}\right)^2\right).
\end{align}
Again writing the KL decomposition between $p$ and $q$ as Eq. \ref{eq:kl_decomposition}, we note that the entropy term $*$ is identical to that of Section \ref{sec:subgaussian_kl}:
\begin{align*}
    \int_{-\infty}^\infty p(x) \log p(x) dx &= -\log(2b_1).
\end{align*}

Turning to the cross-entropy term $**$:
\begin{align*}
    -\int_{-\infty}^\infty p(x) \log q(x) dx 
    &= \int_{\mu_1 - b_1}^{\mu_1 + b_1} \frac{1}{2b_1} \left(\log(\sigma\sqrt{2\pi}) + \frac{1}{2\sigma^2} (x - \mu_2)^2\right) dx \\
    &= \log(\sigma\sqrt{2\pi}) + \int_{\mu_1 - b_1}^{\mu_1 + b_1} \frac{1}{4b_1\sigma^2} (x - \mu_2)^2 dx \\
    &= \log(\sigma\sqrt{2\pi}) + \frac{1}{4b_1\sigma^2} \left(\frac{1}{3} x^3 - \mu_2 x^2 + \mu_2^2 x \right) \Bigg\vert_{\mu_1 - b_1}^{\mu_1 + b_1} \\
    &= \log(\sigma\sqrt{2\pi}) + \frac{1}{4b_1\sigma^2} \left[2 b_1 \left( (\mu_1 - \mu_2)^2 + \frac{1}{3}b_1^2 \right) \right]
\end{align*}

Combining terms, we obtain the KL divergence
\begin{equation}
    KL(p(x)||q(x)) = \frac{1}{2}\left(\frac{1}{\sigma^2}  (\mu_1 - \mu_2)^2 + \log \frac{\pi}{6} + 1 \right),
\end{equation}
where we note that $b_1 = \sqrt{3} \sigma$.

\subsection{KL Divergence Between a Laplace Distribution and a Bounded Uniform Distribution}
\label{sec:laplace_unif_kl}
Let $p$ and $q$ denote the density functions of the $\text{Uniform}([\mu_1 -b_1, \mu_1 + b_1])$ and $\text{Laplace}(\mu_2, b_2)$ distributions, respectively. Then
\begin{align}
    p(x) &= \mathbbm{1}_{x \in [\mu_1 - b_1, \mu_1 + b_1]}\frac{1}{2b_1} \\
    q(x) &= \frac{1}{2b_2} \exp\left(-\frac{|x - \mu_2|}{b_2}\right).
\end{align}
We once again write the KL decomposition between $p$ and $q$ as Eq. \ref{eq:kl_decomposition}, and begin with the entropy term $*$:
\begin{align*}
    \int_{-\infty}^\infty p(x) \log p(x) dx &= -\int_{\mu_1 - b_1}^{\mu_1 + b_1} \frac{\log(2b_1)}{2b_1} dx \\
    &= \frac{\log(2b_1)}{2b_1}(\mu_1 - b_1) - \frac{\log(2b_1)}{2b_1}(\mu_1 + b_1) \\
    &= -\log(2b_1).
\end{align*}

Turning to the cross-entropy term $**$:
\begin{align*}
    -\int_{-\infty}^\infty p(x) \log q(x) dx 
    &= \int_{\mu_1 - b_1}^{\mu_1 + b_1} \frac{1}{2b_1} \left(\log(2b_2) + \frac{|x - \mu_2|}{b_2}\right) dx \\
    &= \log(2b_2) + \int_{\mu_1 - b_1}^{\mu_1 + b_1} \frac{1}{2b_1} \frac{|x - \mu_2|}{b_2} dx.
\end{align*}

Considering the case where $\mu_2 < \mu_1 - b_1$, the above integral reduces to
\begin{align*}
    \int_{\mu_1 - b_1}^{\mu_1 + b_1} \frac{1}{2b_1b_2} |x - \mu_2| dx &= \int_{\mu_1 - b_1}^{\mu_1 + b_1} \frac{1}{2b_1b_2} x - \mu_2 dx \\
    &= \frac{1}{2b_1b_2} \left(\frac{1}{2}x^2 - \mu_2x
    \right) \Bigg\vert_{\mu_1 - b_1}^{\mu_1 + b_1} \\
    &= \frac{1}{2b_1b_2} \left( 2b_1(\mu_1 - \mu_2) \right),
\end{align*}
whereas the case $\mu_2 < \mu_1 - b_1$ gives
\begin{align*}
    \int_{\mu_1 - b_1}^{\mu_1 + b_1} \frac{1}{2b_1b_2} |x - \mu_2| dx &= \int_{\mu_1 - b_1}^{\mu_1 + b_1} \frac{1}{2b_1b_2} x - \mu_2 dx \\
    &= \frac{1}{2b_1b_2} \left(\mu_2x - \frac{1}{2}x^2 
    \right) \Bigg\vert_{\mu_1 - b_1}^{\mu_1 + b_1} \\
    &= \frac{1}{2b_1b_2} \left( 2b_1(\mu_2 - \mu_1) \right).
\end{align*}

Finally, when $\mu_2 \in [\mu_1 - b_1, \mu_1 + b_1]$, we have
\begin{align*}
    \int_{\mu_1 - b_1}^{\mu_1 + b_1} \frac{1}{2b_1b_2} |x - \mu_2| dx &= \frac{1}{2b_1b_2} \left[ \int_{\mu_1 - b_1}^{\mu_2} (\mu_2 - x) dx + \int_{\mu_2}^{\mu_1 + b_1} \frac{1}{2b_1b_2} (x - \mu_2) dx \right] \\
    &= \frac{1}{2b_1b_2} \left[  \left(\mu_2x - \frac{1}{2}x^2 
    \right) \Bigg\vert_{\mu_1 - b_1}^{\mu_2} + \left(\frac{1}{2}x^2 - \mu_2x\right) \Bigg\vert_{\mu_2}^{\mu_1 + b_1} \right] \\
    &= \frac{1}{2b_1b_2} \left[ \left(\frac{1}{2}(\mu_1 - \mu_2)^2 + \frac{1}{2} b_1^2 + b_1(\mu_2 - \mu_1) \right) + \left(\frac{1}{2}(\mu_1 - \mu_2)^2 + \frac{1}{2} b_1^2 + b_1(\mu_1 - \mu_2) \right) \right] \\
    &= \frac{1}{2b_1b_2} \left[ (\mu_1 - \mu_2)^2 + b_1^2 \right].
\end{align*}

Combining the cases, we obtain the KL divergence
\begin{equation}
    KL(p(x)||q(x)) = 
    \begin{cases}
    \log\frac{b_2}{b_1} + \frac{1}{2b_1b_2} ((\mu_1 - \mu_2)^2 + b_1^2 ) & \hspace{.1in}\mu_2 \in [\mu_1 - b_1, \mu_1 + b_1] \\
    \log\frac{b_2}{b_1} + \frac{1}{b_2}|\mu_1 - \mu_2| & \hspace{.1in}\mu_2 \notin [\mu_1 - b_1, \mu_1 + b_1]
    \end{cases}.
\end{equation}




\subsection{KL Divergence Between Linear Sub-Gaussian Distributions}
\label{sec:subgaussian_kl}
Let $\mathbf{LSG}$ denote a Linear Sub-Gaussian Distribution, and $p$ and $q$ denote the density functions of $\mathbf{LSG}(\mu_1, s_1)$ and $\mathbf{LSG}(\mu_2, s_2)$, respectively. For simplicity we assume that $s_1 = s_2$, as this is the case we consider in our diffusion models. Then
\begin{align}
    p(\mathbf{x}) &= \max\left(0, \frac{1}{s^2}(s - |x - \mu_1|)\right) \\
    p(\mathbf{x}) &= \max\left(0, \frac{1}{s^2}(s - |x - \mu_2|)\right).
\end{align}
We once again write the KL decomposition between $p$ and $q$ as Eq. \ref{eq:kl_decomposition}, and start with the cross-entropy term $**$:
\begin{align*}
    -\int_{-\infty}^\infty p(x) \log q(x) dx 
    &= -\int_{\mu_1 - s}^{\mu_1 + s} \frac{1}{s^2}(s - |x - \mu_1|) \log \left(\frac{1}{s^2}(s - |x - \mu_2|)\right) dx.
\end{align*}

Considering the case where $\mu_1 < \mu_2$, we have
\begin{align*}
    \int_{\mu_1 - s}^{\mu_1 + s} \frac{1}{s^2}&(s - |x - \mu_1|) \log \left(\frac{1}{s^2}(s - |x - \mu_2|)\right) dx \\
    &= \int_{\mu_1 - s}^{\mu_1 + s} \frac{1}{s^2}(s - |x - \mu_1|) \log \left(\frac{1}{s^2}(s - |x - \mu_2|)\right) dx 
\end{align*}

\subsection{Deriving $\mathcal{L}_k$}
\label{sec:ll_bound_derivations}
We use the following lemmas to obtain Eqs. (\ref{eq:l2_ll_bound}) and (\ref{eq:l1_ll_bound}) in Sections \ref{sec:gaussian_dm} and \ref{sec:laplace_dm}. Throughout this section, we will use
\begin{align}
    \mathbf{f}_\theta &= \mathbf{f}(\mathbf{x}, t) - \frac{1}{2} g(t)^2 \nabla_\mathbf{x} \log p(\mathbf{x}, t), \\
    \hat{\mathbf{f}}_\theta &= \mathbf{f}(\mathbf{x}, t) - \frac{1}{2} g(t)^2 \mathbf{s}_\theta(\mathbf{x}, t),
\end{align}
where $\mathbf{f}$ and $g$ are defined as in Eq. (\ref{eq:ito_sde}), to denote the true and learned reverse drift terms described in Eq (\ref{eq:reverse_sde}).

\begin{lem}
\label{thm:normal_limit}
Let $\boldsymbol\Delta_{\mathbf{x}_k}$ be normally distributed, i.e., 
\begin{align}
    p_\theta(\boldsymbol\Delta_{\mathbf{x}_k} | \mathbf{x}_k) &= \mathcal{N}\left(\boldsymbol\Delta_{\mathbf{x}_k}; \hat{\mathbf{f}}_\theta\left(\mathbf{x}_k, t_k\right) \Delta_{t_k}, g\left(t_k\right)^2 \Delta_{t_k}\right), \\
    q(\boldsymbol\Delta_{\mathbf{x}_k} | \mathbf{x}_k) &= \mathcal{N}\left(\boldsymbol\Delta_{\mathbf{x}_k}; \mathbf{f}_\theta\left(\mathbf{x}_k, t_k\right) \Delta_{t_k}, g\left(t_k\right)^2 \Delta_{t_k}\right).
\end{align}
Then
\begin{equation}
\mathcal{L}_k = w_k \mathbb{E}_{\boldsymbol\epsilon \sim q} ||\boldsymbol\epsilon - \boldsymbol\epsilon_\theta(\mathbf{x}_k, t_k)||^2,
\end{equation}
where $w_k := \frac{g(t_k)^2}{2\sigma(t_k)^2} \Delta_{t_k}$.
\end{lem}

\begin{proof}
Plugging in the closed form solution to the KL divergence between two Gaussian distributions into the likelihood lower bound,
\begin{align*}
   \mathcal{L}_k &= KL(p_\theta(\boldsymbol\Delta_{\mathbf{x}_k}|\mathbf{x}_k) || q(\boldsymbol\Delta_{\mathbf{x}_k}|\mathbf{x}_k)) \\
    &= \mathbb{E}\left[\frac{||\boldsymbol\mu_{p_\theta,k}(\mathbf{x}_k) - \boldsymbol\mu_{q,k}(\mathbf{x}_k)||^2}{2\sigma^2}\right]. \\
    \intertext{Since $\boldsymbol\mu_{q,k} = (\mathbf{f}(\mathbf{x}_k, t_k) - g(t_k)^2 \nabla \log p(\mathbf{x}_k)) \Delta_{t_k}$, $\boldsymbol\mu_{p,k} = (\mathbf{f}(\mathbf{x}_k, t_k) - g(t_k)^2 \nabla \log p_\theta(\mathbf{x}_k))\Delta_{t_k}$, and $\sigma_{p_\theta,k} = \sigma_{q,k} = g(t_k)\sqrt{\Delta_{t_k}}$, we have}
    \mathcal{L}_k &= \frac{1}{2} \mathbb{E}\left[\frac{||g(t_k)^2 \nabla \log p(\mathbf{x}_k) - g(t_k)^2 \nabla \log p_\theta(\mathbf{x}_k)||^2 \Delta_{t_k}^2}{g(t_k)^2 \Delta_{t_k}}\right] \\
    &= \frac{1}{2} \mathbb{E}\left[g(t_k)^2 ||\nabla \log p(\mathbf{x}_k) - \nabla \log p_\theta(\mathbf{x}_k)||^2 \Delta_{t_k}\right]. \\
    \intertext{Finally, following the parameterization of the score model ($i.e., \boldsymbol\epsilon_\theta(\mathbf{x}, t) = \sigma(t_k) \nabla p_\theta(\mathbf{x}, t)$) in \cite{ho2020denoising}, we may write}
    \mathcal{L}_k &= w_k \mathbb{E}\left[ ||\boldsymbol\epsilon - \boldsymbol\epsilon_\theta(\mathbf{x}_k, t_k)||^2\right], \numberthis \label{eq:l2_derivation}
\end{align*}
where $w_k := \frac{1}{2} g(t_k)^2 \sigma(t_k)^2 \Delta_{t_k}$, and $\sigma(t_k)^2$ is as defined in Lemma \ref{thm:random_walk_mu_sigma}.
\end{proof}

\begin{lem}
\label{thm:laplace_limit}
Let $\boldsymbol\Delta_{\mathbf{x}_k}$ be Laplace distributed, i.e., 
\begin{align}
    p_\theta(\boldsymbol\Delta_{\mathbf{x}_k} | \mathbf{x}_k) &= \text{Laplace}\left(\boldsymbol\Delta_{\mathbf{x}_k}; \hat{\mathbf{f}}_\theta\left(\mathbf{x}_k, t_k\right) \Delta_{t_k}, g\left(t_k\right)^2 \Delta_{t_k}\right), \\
    q(\boldsymbol\Delta_{\mathbf{x}_k} | \mathbf{x}_k) &= \text{Laplace}\left(\boldsymbol\Delta_{\mathbf{x}_k}; \mathbf{f}_\theta\left(\mathbf{x}_k, t_k\right) \Delta_{t_k}, g\left(t_k\right)^2 \Delta_{t_k}\right).
\end{align}
Then, letting $\mathbf{r}_k := \boldsymbol\epsilon - \boldsymbol\epsilon_\theta(\mathbf{x}_k, t_k)$,
\begin{equation}
\mathcal{L}_k = \exp\left(-w_k||\mathbf{r}_k||_1\right) - 1 + w_k||\mathbf{r}_k||_1,
\end{equation}
where $w_k := \frac{g(t_k)}{\sigma(t_k)}\sqrt{\Delta_{t_k}}$..
\end{lem}

\begin{proof}
Plugging in the closed form solution to the KL divergence between two Laplace distributions into the likelihood lower bound (Appendix \ref{sec:laplace_kl}),
\begin{align*}
    \mathcal{L}_k &= KL(p(\boldsymbol\Delta_{\mathbf{x}_k}|\mathbf{x}_k) || q(\boldsymbol\Delta_{\mathbf{x}_k}|\mathbf{x}_k)) \\
    &= \exp\underbrace{\left(\frac{-||\mu_{p_\theta,k} - \mu_{q,k}||_1}{\sigma_p}\right)}_{\mathbf{d}_k} - 1 + \underbrace{\frac{||\mu_{p_\theta,k} - \mu_{q, k}||_1}{\sigma_p}}_{\mathbf{d}_k} \\
\end{align*}
Observe that $\mathbf{d}_k$ can be simplified as
\begin{align*}
    \mathbf{d}_k &= \frac{||\mu_{p_\theta,k} - \mu_{q, k}||_1}{\sigma_p} \\
    &= \frac{g(t_k)^2|| \nabla_\mathbf{x} \log p(\mathbf{x}_k) - \nabla_\mathbf{x} \log p_\theta(\mathbf{x}_k)||_1}{g(t_k)\sqrt{\Delta_{t_k}}} \\
    &= \frac{\sqrt{\Delta_{t_k}} g(t_k)|| \boldsymbol\epsilon - \boldsymbol\epsilon_\theta(\mathbf{x}_k, t_k)||_1}{\sigma(t_k)} \\
    &= v_k || \boldsymbol\epsilon - \boldsymbol\epsilon_\theta(\mathbf{x}_k, t_k)||_1, \numberthis \label{eq:l1_derivation}
\end{align*}
where $v_k := \frac{g(t_k)}{\sigma(t_k)}\sqrt{\Delta_{t_k}}$. Therefore,
\begin{align*}
    \mathcal{L}_k = \left(-v_k||\boldsymbol\epsilon - \boldsymbol\epsilon_\theta(\mathbf{x}_k, t_k)||_1\right) - 1 + v_k||\boldsymbol\epsilon - \boldsymbol\epsilon_\theta(\mathbf{x}_k, t_k)||_1.
\end{align*}
\end{proof}

\begin{lem}
\label{thm:gauss_unif_limit}
Let $p_\theta(\boldsymbol\Delta_{\mathbf{x}_k}|\mathbf{x}_{k+1})$ be normally distributed and $q(\boldsymbol\Delta_{\mathbf{x}_k}|\mathbf{x}_{k+1})$ be the uniform distribution on the interval $[\mu_1 - \sqrt{3}\sigma, \mu_1 + \sqrt{3}\sigma]$, i.e., 
\begin{align}
    p_\theta(\boldsymbol\Delta_{\mathbf{x}_k} | \mathbf{x}_k) = \mathcal{N}\left(\boldsymbol\Delta_{\mathbf{x}_k}; \hat{\mathbf{f}}_\theta\left(\mathbf{x}_k, t_k\right) \Delta_{t_k}, g\left(t_k\right)^2 \Delta_{t_k}\right) \\
    q(\boldsymbol\Delta_{\mathbf{x}_k} | \mathbf{x}_k) = \text{Uniform}\left(\boldsymbol\Delta_{\mathbf{x}_k}; \mathbf{f}_\theta\left(\mathbf{x}_k, t_k\right) \Delta_{t_k}, g\left(t_k\right)^2 \Delta_{t_k}\right).
\end{align}
Then
\begin{equation}
\mathcal{L}_k = w_k \mathbb{E}_{\boldsymbol\epsilon} ||\boldsymbol\epsilon - \boldsymbol\epsilon_\theta(\mathbf{x}_k, t_k)||^2 + C,
\end{equation}
where $w_k := \frac{g(t_k)}{\sigma(t_k)}\sqrt{\Delta_{t_k}}$ and $C = \frac{1}{2}\left(1 + \log \frac{\pi}{6}\right)$.
\end{lem}

\begin{proof}
    Plugging in the closed form solution to the KL divergence between a Gaussian distribution and a Uniform distribution (Appendix \ref{sec:gaussian_unif_kl}), we have
    \begin{align*}
    \mathcal{L}_k &= KL(p(\boldsymbol\Delta_{\mathbf{x}_k}|\mathbf{x}_k) || q(\boldsymbol\Delta_{\mathbf{x}_k}|\mathbf{x}_k)) \\
    &= \mathbb{E}\left[\frac{||\boldsymbol\mu_{p_\theta,k}(\mathbf{x}_k) - \boldsymbol\mu_{q,k}(\mathbf{x}_k)||^2}{2\sigma^2}\right] + C,
\end{align*}
where $C = \frac{1}{2}\left(1 + \log \frac{\pi}{6}\right)$. Since the expectation is the same as Eq. \ref{eq:l2_derivation} in Theorem \ref{thm:normal_limit}, we are done.
\end{proof}

\begin{lem}
\label{thm:laplace_unif_limit}
Let $p_\theta(\boldsymbol\Delta_{\mathbf{x}_k}|\mathbf{x}_{k+1})$ be Laplace distributed and $q(\boldsymbol\Delta_{\mathbf{x}_k}|\mathbf{x}_{k+1})$ be the uniform distribution on the interval $[\mu_1 - \sqrt{3}\sigma, \mu_1 + \sqrt{3}\sigma]$, i.e., 
\begin{align}
    p_\theta(\boldsymbol\Delta_{\mathbf{x}_k} | \mathbf{x}_k) = \text{Laplace}\left(\boldsymbol\Delta_{\mathbf{x}_k}; \hat{\mathbf{f}}_\theta\left(\mathbf{x}_k, t_k\right) \Delta_{t_k}, g\left(t_k\right)^2 \Delta_{t_k}\right) \\
    q(\boldsymbol\Delta_{\mathbf{x}_k} | \mathbf{x}_k) = \text{Uniform}\left(\boldsymbol\Delta_{\mathbf{x}_k}; \mathbf{f}_\theta\left(\mathbf{x}_k, t_k\right) \Delta_{t_k}, g\left(t_k\right)^2 \Delta_{t_k}\right).
\end{align}
Then
\begin{equation}
\mathcal{L}_k = 
\begin{cases}
w_k \mathbb{E}_{\boldsymbol\epsilon} ||\boldsymbol\epsilon - \boldsymbol\epsilon_\theta(\mathbf{x}_k, t_k)||^2 + \frac{1}{2} & \mu_2 \in [\mu_1 - \sqrt{3}\sigma, \mu_1 + \sqrt{3}\sigma] \\
v_k \mathbb{E}_{\boldsymbol\epsilon} ||\boldsymbol\epsilon - \boldsymbol\epsilon_\theta(\mathbf{x}_k, t_k)||^1 & \mu_2 \notin [\mu_1 - \sqrt{3}\sigma w_k, \mu_1 + \sqrt{3}\sigma w_k] \\
\end{cases},
\end{equation}
where $w_k := \frac{g(t_k)}{\sigma(t_k)}\sqrt{\Delta_{t_k}}$ and $v_k := \sqrt{w_k}$.
\end{lem}

\begin{proof}
    Plugging in the closed form solution to the KL divergence between a Laplace distribution and a Uniform distribution (Appendix \ref{sec:laplace_unif_kl}), we have
    \begin{align*}
    \mathcal{L}_k &= \begin{cases}
    \mathbb{E} \left[\frac{||\boldsymbol\mu_{p_\theta,k}(\mathbf{x}_k) - \boldsymbol\mu_{q,k}(\mathbf{x}_k)||^2}{\sigma^2}\right] + \frac{1}{2} & \mu_2 \in [\mu_1 - \sqrt{3}\sigma, \mu_1 + \sqrt{3}\sigma] \\
    \mathbb{E} \left[\frac{||\boldsymbol\mu_{p_\theta,k}(\mathbf{x}_k) - \boldsymbol\mu_{q,k}(\mathbf{x}_k)||^1}{\sigma} \right] & \mu_2 \notin [\mu_1 - \sqrt{3}\sigma, \mu_1 + \sqrt{3}\sigma] \\
    \end{cases}.
\end{align*}
We observe that the event $\mu_2 \in [\mu_1 - \sqrt{3}\sigma, \mu_1 + \sqrt{3}\sigma]$ is equivalent to the event $\boldsymbol\epsilon(\mathbf{x}, t) \in [\boldsymbol\epsilon - \sqrt{3}\sigma w_k, \mu_1 + \sqrt{3}\sigma w_k]$. Using Eqs. \ref{eq:l2_derivation} and \ref{eq:l1_derivation} from Theorems \ref{thm:normal_limit} and \ref{thm:laplace_limit} respectively, we are done. 
\end{proof}

\section{Proofs}

\subsection{Simple Properties of Structured Random Walks}
We show several immediate properties of structured random walks discussed in Section \ref{sec:srw}.

\srwmoments*

\begin{proof}
We first show the derivation for $\mathbb{E} [\mathbf{x}_k]$. Observe that
\begin{align*}
    \mathbb{E} [\mathbf{x}_k] 
    &= \mathbb{E} [\mathbf{x}_{k-1} + \boldsymbol\Delta_{\mathbf{x}_k}] \\
    &= \mathbb{E} \left[\mathbf{x}_{k-1} + \mathbf{f}\left( \mathbf{x}_{k-1}, t_k\right) \Delta_{t_k} + g\left(t_k\right)\mathbf{z}_k\sqrt{\Delta_{t_k}}\right] \\
    &= \mathbb{E}[\mathbf{x}_{k - 1}]\left(1 + \beta\left(t_k\right) \Delta_{t_k}\right).
\end{align*}
Applying this operation $k - 1$ more times, we obtain
\begin{equation*}
    \mathbb{E} [\mathbf{x}_k] = \mathbb{E} [\mathbf{x}_0] \prod_{i=1}^k \left(1 + \beta\left(t_i\right) \Delta_{t_k}\right).
\end{equation*}
Turning to $\Var(\mathbf{x}_k)$, we first note that
\begin{align*}
    \mathbb{E} [\mathbf{x}_k^2] 
    &= \mathbb{E}[(\mathbf{x}_{k-1} + \boldsymbol\Delta_{\mathbf{x}_k})^2] \\
    &= \mathbb{E}[\mathbf{x}_{k-1}^2] + \mathbb{E}[\boldsymbol\Delta_{\mathbf{x}_k}^2] + 2\mathbb{E}[\mathbf{x}_{k-1}\boldsymbol\Delta_{\mathbf{x}_k}]
\end{align*}
where
\begin{align*}
    \mathbb{E}[\mathbf{x}_{k-1}\boldsymbol\Delta_{\mathbf{x}_k}] 
    &= \mathbb{E}\left[\mathbf{x}_{k-1} \left( \mathbf{f}\left( \mathbf{x}_{k-1}, t_k\right) \Delta_{t_k} + g\left(t_k\right)\mathbf{z}_k\sqrt{\Delta_{t_k}} \right) \right] \\
    &= \beta\left(t_k\right) \Delta_{t_k}\mathbb{E}\left[\mathbf{x}_{k-1}^2\right] + g\left(t_k\right) \sqrt{\Delta_{t_k}} \mathbb{E}\left[\mathbf{x}_{k-1} \mathbf{z}_k \right] \\
    &= \beta\left(t_k\right) \Delta_{t_k}\mathbb{E}\left[\mathbf{x}_{k-1}^2\right]
\end{align*}
and
\begin{align*}
    \mathbb{E}[\boldsymbol\Delta_{\mathbf{x}_k}^2] 
    &= \mathbb{E}\left[\left( \mathbf{f}\left( \mathbf{x}_{k-1}, t_k\right) \Delta_{t_k}\right)^2 + \left(g\left(t_k\right)\mathbf{z}_k\sqrt{\Delta_{t_k}} \right)^2 + 2\left( \mathbf{f}\left( \mathbf{x}_{k-1}, t_k\right) \Delta_{t_k}\right)\left(g\left(t_k\right)\mathbf{z}_k\sqrt{\Delta_{t_k}} \right) \right] \\
    &= \beta\left(t_k\right)^2 \Delta_{t_k}^2 \mathbb{E}\left[\mathbf{x}_{k-1}^2\right] + g\left(t_k\right)^2 \Delta_{t_k} \mathbb{E}\left[\mathbf{z}_k^2 \right] + \beta\left(t_k\right) g\left(t_k\right) \Delta_{t_k}^\frac{3}{2} \mathbb{E}[\mathbf{x}_{k-1}\mathbf{z}_k]\\
    &= \beta\left(t_k\right)^2 \Delta_{t_k}^2 \mathbb{E}\left[\mathbf{x}_{k-1}^2\right] + g\left(t_k\right)^2 \Delta_{t_k}
\end{align*}
Putting things together, we have
\begin{align*}
    \mathbb{E} [\mathbf{x}_k^2] 
    &= \mathbb{E}[(\mathbf{x}_{k-1} + \boldsymbol\Delta_{\mathbf{x}_k})^2] \\
    &= \mathbb{E}[\mathbf{x}_{k-1}^2] + \mathbb{E}[\boldsymbol\Delta_{\mathbf{x}_k}^2] + 2\mathbb{E}[\mathbf{x}_{k-1}\boldsymbol\Delta_{\mathbf{x}_k}] \\
    &= \mathbb{E}[\mathbf{x}_{k-1}^2] \left(1 + \beta\left(t_k\right)^2 \Delta_{t_k}^2  + 2\beta\left(t_k\right) \Delta_{t_k} \right) + g\left(t_k\right)^2 \Delta_{t_k} \\
    &= \mathbb{E}[\mathbf{x}_{k-1}^2] \left(1 + \beta\left(t_k\right) \Delta_{t_k} \right)^2 + g\left(t_k\right)^2 \Delta_{t_k}.
\end{align*}
This gives, by induction,
\begin{align*}
    \mathbb{E} [\mathbf{x}_k^2] &= \mathbb{E}[\mathbf{x}_0^2] \prod_{i=1}^k \left(1 + \beta\left(t_i\right) \Delta_{t_k} \right)^2 + \sum_{j=1}^k \prod_{i=j + 1}^k \left(1 + \beta\left(t_i\right) \Delta_{t_i} \right)^2 g\left(t_j\right)^2 \Delta_{t_j}.
\end{align*}
Finally, we can write
\begin{align*}
    \Var(\mathbf{x}_k) 
    &= \mathbb{E} [\mathbf{x}_k^2] -  \mathbb{E} [\mathbf{x}_k]^2\\
    &= \Var(\mathbf{x}_0) \prod_{i=1}^k \left(1 + \beta\left(t_i\right) \Delta_{t_k} \right)^2 + \sum_{j=1}^k \prod_{i=j + 1}^k \left(1 + \beta\left(t_i\right) \Delta_{t_i} \right)^2 g\left(t_j\right)^2 \Delta_{t_j}.
\end{align*}
Assuming that $\Var(\mathbf{x}_0) = 0$, we now have
\begin{equation}
    \boldsymbol\sigma_k^2 := \Var(\mathbf{x}_k) = \sum_{j=1}^k \prod_{i=j + 1}^k \left(1 + \beta\left(t_i\right) \Delta_{t_i} \right)^2 g\left(t_j\right)^2 \Delta_{t_j}.
\end{equation}
\end{proof}

\subsection{Deriving Previous Methods in Our Framework}
\normaldiffusions*

\paragraph{Denoising Diffusion Probabilistic Models} We first examine the forward processes in \cite{ho2020denoising} and \cite{sohl2015deep}, which have the forward Markov chain
\begin{equation}
    p_\theta(\mathbf{x}_{k + 1} | \mathbf{x}_{k}) = \mathcal{N}(\mathbf{x}_{k + 1}; \sqrt{1 - \beta_k} \mathbf{x}_{k}, \beta_k \mathbf{I}),
    \label{eq:ddpm_forward}
\end{equation}
and thus that $\mathbf{x}_{k + 1}$ may be written in terms of $\mathbf{x}_k$ as
\begin{equation}
\mathbf{x}_{k + 1} = \sqrt{1 - \beta_k} \mathbf{x}_k + \sqrt{\beta_k} \boldsymbol\epsilon,
\end{equation}
where $\boldsymbol\epsilon \sim \mathcal{N}(\mathbf{0}, \mathbf{I})$. Subtracting $\mathbf{x}_k$ from both sides and leveraging the fact that $\mathbf{x}_k$ and $\mathbf{x}_{k-1}$ are both normally distributed, we obtain
\begin{equation}
\boldsymbol\Delta_{\mathbf{x}_{k}} = (\sqrt{1 - \beta_k} - 1) \mathbf{x}_k + \sqrt{\beta_k} \boldsymbol\epsilon.
\label{eq:ddpm_derive}
\end{equation}
Now, we see that we can clearly write Eq. (\ref{eq:ddpm_derive}) as a structured random walk (Eq. \ref{eq:structured_random_walk}). Applying Theorem \ref{thm:random_walk_mu_sigma}, we have that
\begin{equation}
    \bar{\alpha}_k = \prod_{i=1}^k (\sqrt{1 - \beta_i}) \hspace{1cm} \bar{\gamma}_k = \sum_{i=1}^k \left(\frac{\bar{\alpha}_k}{\bar{\alpha}_{i + 1}}\right)^2 \beta_i.
\end{equation}
This converges numerically to the form given in \cite{ho2020denoising}
\begin{equation}
    p(\mathbf{x}_k|\mathbf{x}_0) = \mathcal{N}(\mathbf{x}_k; \bar{\alpha}_k \mathbf{x}_0, (1 - \bar{\alpha}^2) \mathbf{I}).
\end{equation}

\paragraph{Variational Diffusion Models} We can obtain a similar closed form solution for the forward process in \cite{kingma2021variational}. The sampling chain of the process can be written as
\begin{equation}
    p(\mathbf{x}_t) = \mathcal{N}\left(\mathbf{x}_{k+1}; \frac{\alpha_{k+1}}{\alpha_{k}} \mathbf{x}_k, \sigma^2_{k+1} - \left(\frac{\alpha_{k+1}}{\alpha_{k}}\right)\sigma_k^2\right),
    \label{eq:vdm_steps}
\end{equation}
where $\alpha_k$ and $\sigma_k$ are related to each other by a monotonic function $\gamma(t)$
\begin{align}
    \alpha_k^2 &= \text{sigmoid}(-\gamma(t)), \\
    \sigma_k^2 &= \text{sigmoid}(\gamma(t)).
\end{align}
According to Eq. \ref{eq:vdm_steps}, $\mathbf{x}_{k+1}$ can be written in terms of $\mathbf{x}_k$ as
\begin{equation}
    \mathbf{x}_{k+1} = \frac{\alpha_{k+1}}{\alpha_{k}} \mathbf{x}_k + \left( \sigma^2_{k+1} - \left(\frac{\alpha_{k+1}}{\alpha_{k}}\right)\sigma_k^2 \right) \boldsymbol\epsilon,
\end{equation}
where $\boldsymbol\epsilon \sim \mathcal{N}(\mathbf{0}, \mathbf{I})$. Subtracting off $\mathbf{x}_k$ on both sides, we obtain
\begin{equation}
    \boldsymbol\Delta_{\mathbf{x}_{k+1}} = \left(\frac{\alpha_{k+1}}{\alpha_{k}} - 1\right) \mathbf{x}_k + \left( \sigma^2_{k+1} - \left(\frac{\alpha_{k+1}}{\alpha_{k}}\right)\sigma_k^2 \right) \boldsymbol\epsilon.
\end{equation}
Now we can once again apply Theorem \ref{thm:random_walk_mu_sigma}, and see that
\begin{equation}
    \bar{\alpha}_k = \prod_{i=1}^k \frac{\alpha_{i+1}}{\alpha_{i}} = \alpha_k
\end{equation}
and
\begin{align}
    \bar{\gamma}_k &= \sum_{i=1}^k \frac{\alpha_k}{\alpha_{i}} \left( \sigma^2_{i+1} - \left(\frac{\alpha_{i}}{\alpha_{i - 1}}\right)\sigma_i^2 \right) \\
    &= \sum_{i=1}^k \frac{\alpha_k}{\alpha_i} \sigma^2_{i+1} - \frac{\alpha_{k}}{\alpha_{i-1}}\sigma_i^2 \\
    &= \sigma_{k}^2,
\end{align}
which agrees with the marginals in \cite{kingma2021variational}.

\subsection{Regularity Conditions}
\label{sec:regularity}
To show our main result, we state the following regularity conditions. Assumptions \ref{asm:lipschitz} and \ref{asm:lin_growth} are standard for finite-step discretizations of SDEs \cite{sarkka2019applied}. Assumption \ref{asm:boundedness} simplifies the subsequent proof for tightness.
\begin{asm}[$\mathbf{f}$ and $g$ are Lipschitz]
\label{asm:lipschitz}
There exists $K > 0$ such that, for any $\mathbf{x}, \mathbf{y} \in \mathbb{R}^d$ and $t, s \in [0, 1]$
\begin{equation}
    ||\mathbf{f}(\mathbf{x}) - \mathbf{f}(\mathbf{y})|| \leq K ||\mathbf{x} - \mathbf{y}||, \hspace{.25in} \text{and} \hspace{.25in} |g(t) - g(s)| \leq K |t - s|.
\end{equation}
\end{asm}
\begin{asm}[Linear growth of $\mathbf{f}$ and $g$]
\label{asm:lin_growth} There exists $K > 0$ such that, for any $\mathbf{x}\in \mathbb{R}^d$ and $t \in [0, 1]$
\begin{equation}
    ||\mathbf{f}(\mathbf{x})|| \leq K (1 + ||\mathbf{x}||), \hspace{.25in} \text{and} \hspace{.25in} |g(t)| \leq K (1 + |t|).
\end{equation}
\end{asm}
\begin{asm}[Integrability of $\mathbf{z}_k$.]
\label{asm:boundedness}
There exists $K \in \mathbb{R}$ such that
\begin{equation}
    \mathbb{E} [||\mathbf{z}_k||^4]) < K.
\end{equation} 
\end{asm}

\subsection{Main Result}
Our theorem below can be seen as a generalization of Donsker's Invariance Principle, and certain parts of the proof resembles that of the original theorem. Differences appear where we can no longer rely on the independence of the increments $\boldsymbol\Delta_{\mathbf{x}_k}$, which is heavily utilized in the original proof. By exploiting the structural properties of Definition \ref{defn:srw}, we can decompose $\mathbf{x}_k$ into a set of auxiliary processes with the same limit, which we can show to converge to $\mathbf{X}$ with techniques borrowed from the strong convergence of SDE solvers and central limit theorems.

\invariance*

\begin{proof}

Using Eq. \ref{eq:structured_random_walk} we may define the continuous-time extension of $\mathbf{x}_k$ as the process
\begin{equation}
    \mathbf{x}_T(t) = \mathbf{x}_0 + \sum_{i=1}^{\lfloor t * T \rfloor} \boldsymbol\Delta^{(T)}_{\mathbf{x}_i} + (t * T  - \lfloor t * T \rfloor) \boldsymbol\Delta^{(T)}_{\mathbf{x}_{\lceil t * T \rceil}},
    \label{eq:interpolation}
\end{equation}
which is produced by linearly interpolating between the iterates of the random walk. We write the increments 
\begin{equation}
    \boldsymbol\Delta^{(T)}_{\mathbf{x}_i} := \mathbf{x}_{k + 1} - \mathbf{x}_k = \mathbf{f}(\mathbf{x}_k, t_k) \Delta^{(T)}_{t_k} + g(t_k) \sqrt{\Delta^{(T)}_{t_k}} Z^{(T)}_k
\end{equation}with the superscript $(T)$ to emphasize its dependence on $T$. We show convergence by invoking the following theorem.

\begin{thm} (Theorem 13.1 from \cite{billingsley2013convergence}.)
\label{thm:billingsley}
Let $\{\mathbf{x}_T\}, \mathbf{x}$ be processes (with associated probability measures $\{\mathcal{P}_T\}, \mathcal{P}$) such that $\mathbf{x}_T$ converges to $\mathbf{x}$ in finite dimensional distributions (f.d.d.), i.e., for any $k$ time steps $t_1, t_2, \dots, t_k$,
\begin{equation}
    (\mathbf{x}_T(t_1), \mathbf{x}_T(t_2), \dots, \mathbf{x}_T(t_k)) \xlongrightarrow{\mathcal{D}} (\mathbf{x}(t_1), \mathbf{x}(t_2), \dots, \mathbf{x}(t_k)).
\end{equation}
If $\{\mathcal{P}_T\}$ are also tight, then $\mathbf{x}_T \Rightarrow_T \mathbf{x}$.
\end{thm}

Theorem \ref{thm:billingsley} relates the pointwise weak convergence (of a sequence of marginals of a process) on a finite set of  points to weak convergence of the path measures. This is made possible by Prohorov's theorem, which connects tightness to relative compactness. Thus, to show convergence, we must show two conditions are satisfied: 1) convergence in f.d.d., and 2) tightness of the associated sequence of measures. These are given by the following two lemmas.

\begin{lem}
\label{thm:tightness}
The sequence of measures $\{P_{\mathbf{x}_k}\}_{k=1}^T$ corresponding to the \textbf{structured random walk} $\{\mathbf{x}_k\}_{k=1}^T$ is tight.
\end{lem}

\begin{lem}
\label{thm:fdd}
The continuous-time random walk interpolation $\mathbf{x}_T$ converges in finite dimensional distributions (f.d.d.) to the diffusion process (i.e., solution to Eq. (\ref{eq:ito_sde}) $\mathbf{x}$.
\end{lem}

Combining Theorem \ref{thm:billingsley} with Lemmas \ref{thm:tightness} and \ref{thm:fdd}, we obtain our result.
\end{proof}

\begin{proof} (of Lemma \ref{thm:tightness})

The result can be obtained via Kolmogorov's tightness criterion, which provides the following sufficient condition for tightness:

\begin{equation}
    \label{eq:ktc}
    \sup_n \mathbb{E} \left[ ||\mathbf{x}_n(s) - \mathbf{x}_n(t)||^p \right] \leq C |s - t|^{1 + \epsilon}, \hspace{.1in} \text{for some } \epsilon > 0, p \geq 1 + \epsilon.
\end{equation}

We shall demonstrate Eq. \ref{eq:ktc} for $\epsilon=1$, $p=4$. For any $s, t \in [0, T]$, choose $k, \ell$ such that
\begin{equation}
    s \in \Bigg[\frac{k - 1}{n}, \frac{k}{n} \Bigg) \hspace{.1in} \text{and} \hspace{.1in} t \in \Bigg[ \frac{\ell - 1}{n}, \frac{\ell}{n} \Bigg).
\end{equation}

First, observe that, applying Definition \ref{defn:srw}, Assumption \ref{asm:lin_growth}, and the fact that $\mathbf{z}_k \in \mathcal{L}^4 \implies \mathbf{E}[||\mathbf{z}_k||^4] \leq M$ for some $M \in \mathbb{R}$ and all $k \in \{0, \dots, n\}$,
\begin{align*}
    \mathbb{E} [||\mathbf{x}_n(s - \Delta t) - \mathbf{x}_n(s)||_4] 
    &= \mathbb{E} [||\boldsymbol\Delta_{\mathbf{x}_k}||_4] \\
    &= \mathbb{E} \left[ \left|\left|\mathbf{f}(\mathbf{x}_k) \Delta t + g\left(\frac{k}{n}\right) \mathbf{z}_k \sqrt{\Delta t} \right|\right|_4\right] \\
    &\leq \Delta t\mathbb{E} [||\mathbf{f}(\mathbf{x}_k)||_4] + \sqrt{\Delta t} \left|g\left(\frac{k}{n}\right)\right| \mathbb{E}[||\mathbf{z}_k ||_4] ]\\
    &\leq \Delta t\mathbb{E} [K(1 + ||\mathbf{x}_k||_4)] + \sqrt{\Delta t}KM \left(1 + \frac{k}{n}\right)\\
    &\leq C \sqrt{\Delta t}, \numberthis \label{eq:delta_bound}
\end{align*}
where $C_1 \leq \mathcal{O}(\sqrt{\Delta t})$. 

We will bound $\mathbb{E} \left[ ||\mathbf{x}_n(s) - \mathbf{x}_n(t)||^4 \right]$ in three regimes:

\textbf{Case 1:} $k = \ell$
\begin{align*}
    \mathbb{E} \left[ ||\mathbf{x}_n(s) - \mathbf{x}_n(t)||_4 \right] &= \mathbb{E} \left[ \left| \left| \mathbf{x}_0 + \sum_{i=1}^{k - 1} \boldsymbol\Delta_{\mathbf{x}_i} + (ns - k) \boldsymbol\Delta_{\mathbf{x}_k} - \mathbf{x}_0 - \sum_{i=1}^{\ell - 1} \boldsymbol\Delta_{\mathbf{x}_i} - (nt - \ell) \boldsymbol\Delta_\ell \right|\right|_4 \right] \\
    &= \mathbb{E} \left[ \left|\left| (ns - k) \boldsymbol\Delta_{\mathbf{x}_k} - (nt - k) \boldsymbol\Delta_{\mathbf{x}_k} \right|\right|_4 \right] \\
    &\leq (n|t - s|) \mathbb{E} \left[ ||\boldsymbol\Delta_{\mathbf{x}_k}||_4 \right] \\
    &\leq C_1 \sqrt{n} |t - s|,
\end{align*}
where we used Eq. \ref{eq:delta_bound} the fact that $\Delta t = \frac{1}{n}$. Finally, since $k = \ell \implies |t - s| \leq n^{-1}$, we take the fourth power of both sides of the inequality to obtain
\begin{equation}
    \mathbb{E} \left[ ||\mathbf{x}_n(s) - \mathbf{x}_n(t)||^4 \right] \leq C_1 |t - s|^2.
    \label{eq:case1}
\end{equation}

\textbf{Case 2:} $k = \ell + 1$
\begin{align*}
\mathbb{E} \left[ ||\mathbf{x}_n(s) - \mathbf{x}_n(t)||_4 \right] &= \mathbb{E} \left[ \left| \left| \mathbf{x}_0 + \sum_{i=1}^{k - 1} \boldsymbol\Delta_{\mathbf{x}_i} + (ns - k) \boldsymbol\Delta_{\mathbf{x}_k} - \mathbf{x}_0 - \sum_{i=1}^{\ell - 1} \boldsymbol\Delta_{\mathbf{x}_i} - (nt - \ell) \boldsymbol\Delta_\ell \right|\right|_4 \right] \\
&= \mathbb{E} \left[ \left| \left| \boldsymbol\Delta_\ell + (ns - k) \boldsymbol\Delta_{\mathbf{x}_k} - (nt - \ell) \boldsymbol\Delta_\ell \right|\right|_4 \right] \\
&= \mathbb{E} \left[ \left| \left| (ns - k) \boldsymbol\Delta_{\mathbf{x}_k} - (nt - k) \boldsymbol\Delta_\ell \right|\right|_4 \right] \\
&\leq |(ns - nt)| \mathbb{E} \left[ ||\boldsymbol\Delta_{\mathbf{x}_k}||_4 \right] \\
&\leq C_1 \sqrt{n} |t - s|,
\end{align*}
where we again use Eq. \ref{eq:delta_bound} the fact that $\Delta t = \frac{1}{n}$. This time, we have that $|t - s| \leq 2n^{-1}$. Therefore,
\begin{equation}
    \mathbb{E} \left[ ||\mathbf{x}_n(s) - \mathbf{x}_n(t)||^4 \right] = 4C_1 |t - s|^2.
    \label{eq:case2}
\end{equation}

\textbf{Case 3:} $k > \ell + 2$
\begin{align*}
    \mathbb{E} [ ||\mathbf{x}_n(s) &- \mathbf{x}_n(t)||_4 ] \\ 
    &\leq \mathbb{E} \Bigg[ \left|\left|\mathbf{x}_n(s) - \mathbf{x}_n\left(\frac{k - 1}{n}\right)\right|\right|_4 + \left|\left|\mathbf{x}_n\left(\frac{k - 1}{n}\right) - \mathbf{x}_n\left(\frac{\ell - 1}{n}\right)\right|\right|_4 \\
    &\hspace{.4in}+ \left|\left|\mathbf{x}_n(t) - \mathbf{x}_n\left(\frac{\ell - 1}{n}\right)\right|\right|_4 \Bigg] \\
    &\leq C_1 \sqrt{s - \frac{k-1}{n}} + \mathbb{E} \left[\left|\left|\mathbf{x}_n\left(\frac{k - 1}{n}\right) - \mathbf{x}_n\left(\frac{\ell - 1}{n}\right)\right|\right|_4 \right] \\
    &\hspace{.4in} + C_1 \sqrt{t - \frac{\ell-1}{n}} \\
    &\leq C_1 \left(\sqrt{s - \frac{k-1}{n}} - \sqrt{t - \frac{\ell-1}{n}}\right) + \underbrace{\mathbb{E} \left[\left|\left|\mathbf{x}_n\left(\frac{k - 1}{n}\right) - \mathbf{x}_n\left(\frac{\ell - 1}{n}\right)\right|\right|_4 \right]}_{(*)}
\end{align*}
Inspecting $(*)$, we can see that
\begin{align*}
\mathbb{E} \left[\left|\left|\mathbf{x}_n\left(\frac{k - 1}{n}\right) - \mathbf{x}_n\left(\frac{\ell - 1}{n}\right)\right|\right|_4 \right] &= \mathbb{E} \left[ \left| \left| \mathbf{x}_0 + \sum_{i=1}^{k - 1} \boldsymbol\Delta_{\mathbf{x}_i} - \mathbf{x}_0 - \sum_{i=1}^{\ell - 1} \boldsymbol\Delta_{\mathbf{x}_i} \right|\right|_4 \right] \\
&= \mathbb{E} \left[ \left| \left| \sum_{i=\ell - 1}^{k - 1} \boldsymbol\Delta_{\mathbf{x}_i} \right|\right|_4 \right] \\
&\leq \sum_{i = \ell - 1}^{k - 1} \mathbb{E} [\left|\left|\mathbf{f}(\mathbf{x}_i) \Delta t\right|\right|] +  \mathbb{E} \left[\left|\left|\sum_{i = \ell - 1}^{k - 1} g\left(\frac{i}{n}\right) \mathbf{z}_i \sqrt{\Delta t} \right|\right|_4\right] \\
&\leq C_1 \sqrt{\frac{k}{n} - \frac{\ell}{n}} + \mathbb{E} \left[\left|\left|\sum_{i = \ell - 1}^{k - 1} g\left(\frac{i}{n}\right) \mathbf{z}_i \sqrt{\Delta t} \right|\right|_4\right]. \numberthis \label{eq:case3_0}
\end{align*}
We make the following observation about the second term in Eq. \ref{eq:case3_0}.
\begin{lem}
    \begin{equation}
        \mathbb{E} \left[\left|\left|\sum_{i = \ell - 1}^{k - 1} g\left(\frac{i}{n}\right) \mathbf{z}_i \sqrt{\Delta t} \right|\right|^4\right] \leq C_2 \left( \frac{k}{n} - \frac{\ell}{n} \right)^2
        \label{eq:induction_lemma_eq2}
    \end{equation}
\end{lem}
\begin{proof}
Letting $\mathbf{W}_i := g\left(\frac{i}{n}\right) \mathbf{z}_i \sqrt{\Delta t}$, the second term in Eq. \ref{eq:case3_0}, taken to the fourth power, can be written as
\begin{equation}
    \mathbb{E} \left[\left|\left|\sum_{i = \ell - 1}^{k - 1} g\left(\frac{i}{n}\right) \mathbf{z}_i \sqrt{\Delta t} \right|\right|^4\right] 
    = \mathbb{E} \left[\left|\left|\sum_{i = \ell - 1}^{k - 1} \mathbf{W}_i\right|\right|^4\right],
    \label{eq:induction_lemma_eq1}
\end{equation}
where $\mathbb{E}[\mathbf{W}_i] = 0$ and $\mathbb{E}[\mathbf{W}_i^2] = g^2\left(\frac{i}{n}\right) \Delta t$. The result will be shown by induction. Separating an element of the sum and then expanding the norm, we can write this term as
    \begin{align*}
    \mathbb{E} \left[\left|\left|\sum_{i = \ell - 1}^{k - 2} \mathbf{W}_i + W_{k - 1}\right|\right|^4\right] &= \mathbb{E} \left[\left|\left|\sum_{i = \ell - 1}^{k - 2} \mathbf{W}_i \right|\right|^4 + \Bigg|\Bigg|W_{k - 1}\Bigg|\Bigg|^4 + \left|\left|\sum_{i = \ell - 1}^{k - 2} \mathbf{W}_i \right|\right|^2 \Bigg|\Bigg|W_{k - 1}\Bigg|\Bigg|^2\right],
    \intertext{where the odd terms containing first moments of $\mathbf{W}_i$ go to zero. Leveraging Assumption \ref{asm:lin_growth} and the fact that $\mathbf{z}_k \in \mathcal{L}^4$ we can further simplify left hand side to}
    \mathbb{E} \left[\left|\left|\sum_{i = \ell - 1}^{k - 2} \mathbf{W}_i + W_{k - 1}\right|\right|^4\right]
    &\leq \mathbb{E} \left[\left|\left|\sum_{i = \ell - 1}^{k - 2} \mathbf{W}_i \right|\right|^4 \right] + M^4 \Bigg( g^4\left(\frac{k - 1}{n}\right) (\Delta t)^2 \Bigg) \\
    &\hspace{.4in} + \left(\sum_{i = \ell - 1}^{k - 2} M^2 g^2\left(\frac{i}{n}\right) \Delta t \right)\Bigg(g^2\left(\frac{k - 1}{n}\right) \Delta t \Bigg) \\
    &\leq \mathbb{E} \left[\left|\left|\sum_{i = \ell - 1}^{k - 2} \mathbf{W}_i \right|\right|^4 \right] + C_2 (k - \ell) (\Delta t)^2.
\end{align*}
Applying this operation $k - \ell - 1$ more times, we obtain our desired result
\begin{align*}
    \mathbb{E} \left[\left|\left|\sum_{i = \ell - 1}^{k - 1} \mathbf{W}_i \right|\right|^4\right] &\leq C_2(k - \ell)^2(\Delta t)^2 \\
    &= C_2 \left( \frac{k}{n} - \frac{\ell}{n} \right)^2.
\end{align*}
\end{proof}

Assembling the parts, we obtain
\begin{align*}
    \mathbb{E} [ ||\mathbf{x}_n(s) - \mathbf{x}_n(t)||^4 ]
    &\leq C_1 \left[\left(s - \frac{k-1}{n}\right)^2 - \left(t - \frac{\ell-1}{n}\right)^2 + \left(\frac{k}{n} - \frac{\ell}{n} \right)^2 \right] + C_2 \left( \frac{k}{n} - \frac{\ell}{n} \right)^2 \\
    &\leq C_3 |t - s|^2. \numberthis \label{eq:case3}
\end{align*}

Finally, we combine Eqs. \ref{eq:case1}, \ref{eq:case2}, and \ref{eq:case3}, which provides a bound that satisfies Kolmogorov's tightness criterion:
\begin{equation}
    \mathbb{E} [ ||\mathbf{x}_n(s) - \mathbf{x}_n(t)||^4 ] \leq \max(4 C_1, C_3) |t - s|^2.
\end{equation}
\end{proof}

\begin{proof} (of Lemma \ref{thm:fdd})

Let us define the auxiliary processes
\begin{align}
    \widetilde{\mathbf{x}}_{T}(t) = \mathbf{x}_0 + \sum_{i=1}^{\lfloor t * T \rfloor} \widetilde{\boldsymbol\Delta}^{(T)}_{\mathbf{x}_i} + (t * T  - \lfloor t * T \rfloor) \widetilde{\boldsymbol\Delta}^{(T)}_{\mathbf{x}_{\lceil t * T \rceil}} \label{eq:interpolation_normal} \\
    \mathbf{x}_{T,S}(t) = \mathbf{x}_0 + \sum_{i=1}^{\lfloor t * T \rfloor} \boldsymbol\Delta^{(T,S)}_{\mathbf{x}_i} + (t * T  - \lfloor t * T \rfloor) \boldsymbol\Delta^{(T,S)}_{\mathbf{x}_{\lceil t * T \rceil}} \label{eq:interpolation_s},
\end{align}
where
\begin{align}
    \widetilde{\boldsymbol\Delta}^{(T)}_{\mathbf{x}_k} = \mathbf{f}(\mathbf{x}_k, t_k) \Delta_{t_k} + g(t_k) \mathbf{W}(\Delta_{t_k}) \\
    \boldsymbol\Delta^{(T,S)}_{\mathbf{x}_k} = \mathbf{f}(\mathbf{x}_k, t_k) \Delta^{(T)}_{t_k} + \sum_{i=0}^{S-1} g(t_k) \sqrt{\Delta^{(T)}_{t_k}} \mathbf{z}^{(S * T)}_{k + Ti},
\end{align}
and $\mathbf{W}(\Delta_{t_k}) := \mathcal{N}(0, I * \Delta_{t_k})$. Slightly overloading our notation and letting
\begin{equation}
    \mathbf{A}(\{t_i\}_{i=1}^k) := (\mathbf{A}(t_1), \dots, \mathbf{A}(t_k))
\end{equation}
for a diffusion process $\mathbf{A}(t)$ evaluated at times $(t_1, t_2, \dots, t_k)$, we may obtain the desired result by observing that
\begin{align}
    &\lim_{T \rightarrow \infty} \text{CDF}[\mathbf{x}_T(\{t_i\}_{i=1}^k)] \\
    = &\lim_{T \rightarrow \infty} \lim_{S \rightarrow \infty} \text{CDF}[\mathbf{x}_T(\{t_i\}_{i=1}^k) + (\mathbf{x}_{T,S}(\{t_i\}_{i=1}^k) - \mathbf{x}_T(\{t_i\}_{i=1}^k))] \\
    = &\lim_{T \rightarrow \infty} \lim_{S \rightarrow \infty} \text{CDF}[\mathbf{x}_{T,S}(\{t_i\}_{i=1}^k)] \\
    = &\lim_{T \rightarrow \infty} \text{CDF}[\widetilde{\mathbf{x}}_T(\{t_i\}_{i=1}^k)] && \text{Lemma \ref{thm:interpolation_convergence}} \\
    = &\hspace{.08in}\text{CDF}[\mathbf{x}(\{t_i\}_{i=1}^k)]. && \text{Lemma \ref{thm:l2_convergence}}
\end{align}

Next, we may interpret $\widetilde{\mathbf{x}}_{T}(t)$ (Eq. \ref{eq:interpolation_normal}) as a variant of $\mathbf{x}_T(t)$ (Eq. (\ref{eq:interpolation}) with "normalized" increments, which can be formally shown to be the limit of $\mathbf{x}_{T,S}(t)$ (Eq. \ref{eq:interpolation_s}) as $S \rightarrow \infty$ by the central limit theorem.
\begin{lem}
\label{thm:interpolation_convergence}
Let $\mathbf{x}_{T,S}(t)$ and $\widetilde{\mathbf{x}}_{T}(t)$ be defined as above. Then $\mathbf{x}_{T,S}(t)$ converges in f.d.d. to $\widetilde{\mathbf{x}}_{T}(t)$.
\end{lem}

Finally, the result $\widetilde{\mathbf{x}}_{T} \stackrel{\text{f.d.d.}}{\longrightarrow} \mathbf{x}$ can be shown via techniques that follow closely to the proof for the strong convergence of SDE solvers. For $i \in \{0, \dots, n\}$ and $t \in [0, T]$ we let
\begin{equation}
    \overline{\mathbf{x}}(t) = \sum_{k=1}^n \mathbf{x}_k \mathbbm{1}_{t \in [t_k, t_{k + 1})}(t) \hspace{.2in}
\end{equation}
be the continuous-time c\`{a}dl\`{a}g extensions of the random walk $\mathbf{x}_k$. Now, $\widetilde{\mathbf{x}}_T$ can also be written as the It\^{o} integral
\begin{equation}
    \widetilde{\mathbf{x}}_T
    = \mathbf{x}(0) + \int_0^t f(\overline{\mathbf{x}}(s)) ds + \int_0^t g\left(\frac{\lfloor s * T \rfloor}{n}\right) dW_s.
\end{equation}
Of course, the solution $\mathbf{x}$ to Eq. \ref{eq:ito_sde} can also be expressed in the similar form
\begin{equation}
    \mathbf{x}(t) = \mathbf{x}(0) + \int_0^t f(\mathbf{x}(s)) ds + \int_0^t g(s) dW_s.
\end{equation}
Now we may state the following lemma.
\begin{lem}
\label{thm:l2_convergence}
Let $\widetilde{\mathbf{x}}_T$ be defined as above and Assumption \ref{asm:lipschitz} hold. Then $\widetilde{\mathbf{x}}_T$ converges to $\mathbf{x}$ in finite dimensional distributions (f.d.d.).
\end{lem}
\end{proof}

\begin{proof} (of Lemma \ref{thm:interpolation_convergence})

Observe that Eq. (\ref{eq:interpolation_normal}) can be seen as the continuous-time interpolation of the random walk
\begin{equation}
    \widetilde{\mathbf{x}}_{k+1}^{(T)} = \widetilde{\mathbf{x}}_k^{(T)} + \mathbf{f}(\widetilde{\mathbf{x}}_k^{(T)}, t_k)\Delta_{t_k} + g(t_k) \sqrt{\Delta_{t_k}} \mathbf{W}(\Delta_{t_k}),
    \label{eq:rw_normal}
\end{equation}
and Eq. (\ref{eq:interpolation_s}) of the random walk
\begin{equation}
    \mathbf{x}_{k+1}^{(T, S)} = \mathbf{x}_k^{(T, S)} + \mathbf{f}(\mathbf{x}_k^{(T, S)}, t_k)\Delta_{t_k} + \sum_{i=1}^S g(t_k) \sqrt{\Delta^{(T)}_{t_k}} \mathbf{z}^{(S * T)}_{k + Ti}.
    \label{eq:rw_s}
\end{equation}
Applying the Central Limit Theorem, we may see that
\begin{align*}
    \mathbf{A}_k := \sum_{i=1}^S g(t_k) \sqrt{\Delta^{(T)}_{t_k}} \mathbf{z}^{(S * T)}_{k + Ti} \stackrel{\mathcal{D}}{\longrightarrow} g(t) \sqrt{\Delta_{t_k}} \mathbf{W}(\Delta_{t_k})
\end{align*}
for each $0 \leq k < T$. We now show our result by recursion. In the base case we have that $\mathbf{x}^{(T,S)}_0 := \widetilde{\mathbf{x}}^{(T)}_0 := \mathbf{x}_0$, so clearly $\mathbf{x}^{(T,S)}_0 \stackrel{\mathcal{D}}{\longrightarrow} \widetilde{\mathbf{x}}^{(T)}_0$. For any subsequent $k + 1 > 0$, we may invoke Slutsky's Theorem on the independent sequences $\mathbf{x}^{(T,S)}_k \stackrel{\mathcal{D}}{\longrightarrow} \widetilde{\mathbf{x}}^{(T)}_k$ and $\mathbf{A}_k \stackrel{\mathcal{D}}{\longrightarrow} g(t) \sqrt{\Delta_{t_k}} \mathbf{W}(\Delta_{t_k})$ to obtain
\begin{equation}
    \mathbf{x}^{(T,S)}_{k + 1} := \mathbf{x}^{(T,S)}_k + \mathbf{A}_k \stackrel{\mathcal{D}}{\longrightarrow} \widetilde{\mathbf{x}}^{(T)}_k + g(t) \sqrt{\Delta_{t_k}} \mathbf{W}(\Delta_{t_k}) =: \widetilde{\mathbf{x}}^{(T)}_{k+1}.
\end{equation}
Therefore, we have that $\mathbf{x}_{k+1}^{(T, S)} \stackrel{\mathcal{D}}{\longrightarrow} \mathbf{x}_{k+1}^{(T)}$ for all $k$. Since Eqs. \ref{eq:interpolation_normal} and \ref{eq:interpolation_s} are purely functions of $t$ and their respective random walks (Eqs. \ref{eq:rw_normal} and \ref{eq:rw_s}), we have our result.
\end{proof}

\begin{proof} (of Lemma \ref{thm:l2_convergence})
Let us define
\begin{equation}
    \boldsymbol\epsilon(t) = \sup_{0 \leq s \leq t} \mathbb{E}\left[ \big|\big|\mathbf{Y}_n(s) - \mathbf{x}(s) \big|\big|^2  \right].
\end{equation}
Recalling the definitions $\overline{\mathbf{x}}(t) := \mathbf{x}_{\lfloor nt \rfloor}$ and $\bar{\mathbf{z}}(t) = \mathbf{z}_{\lfloor nt \rfloor}$, we have
\begin{align*}
    \boldsymbol\epsilon(t) &= \sup_{0 \leq s \leq t} \mathbb{E}\Bigg[\Bigg|\Bigg| \int_0^s [f(\mathbf{x}(u)) - f(\overline{\mathbf{x}}(u))] du + \int_0^s \left[g\left(\frac{\lfloor n \cdot u \rfloor}{n}\right) - g(u)\right] dW_u \Bigg|\Bigg|^2\Bigg] \\
    &\leq 4\sup_{0 \leq s \leq t} \mathbb{E}\Bigg[ \Bigg|\Bigg| \int_0^s [f(\mathbf{x}(u)) - f(\overline{\mathbf{x}}(u))] du \Bigg|\Bigg|^2 + \Bigg|\Bigg| \int_0^s \left[g\left(\frac{\lfloor n \cdot u \rfloor}{n}\right) - g(u)\right] dW_u \Bigg|\Bigg|^2 \Bigg].
    \intertext{Invoking the It\^{o} isometry, Cauchy-Schwarz inequality, and linearity of expectations,}
    \boldsymbol\epsilon(t) &\leq 4\sup_{0 \leq s \leq t} \Bigg( t \mathbb{E} \bigg[ \int_0^s \bigg|\bigg|f(\mathbf{x}(u)) - f(\overline{\mathbf{x}}(u))\bigg|\bigg|^2 du \bigg] + \mathbb{E}\bigg[ \int_0^s \bigg|\bigg| g\left(\frac{\lfloor n \cdot u \rfloor}{n}\right) - g(u) \bigg|\bigg|^2 du \bigg] \Bigg). \\
    \intertext{We now leverage Assumption \ref{asm:lipschitz} to obtain}
    \boldsymbol\epsilon(t) &\leq 4K\sup_{0 \leq s \leq t} \Bigg(t \mathbb{E} \bigg[ \int_0^s \bigg|\bigg|\mathbf{x}(u)) - \overline{\mathbf{x}}(u)\bigg|\bigg|^2 du \bigg] + \mathbb{E}\bigg[ \int_0^s \frac{1}{n^2} du \bigg] \Bigg).
\end{align*}
\begin{align*}
    \intertext{Applying Theorem 4.5.4 in \cite{kloeden1992stochastic} and folding all constants that depend on $T, \mathbb{E}[X_0], K$ into $C$, we have}
    \boldsymbol\epsilon(t) &\leq C \Bigg(\int_0^s \boldsymbol\epsilon(u) du + \Delta t \Bigg), \\
\end{align*}
which, by Gronwall's inequality, results in the bound
\begin{equation}
    \sup_{0 \leq s \leq T} \mathbb{E}\left[ \left|\mathbf{Y}_n(s) - \mathbf{x}(s) \right|^2  \right] = \boldsymbol\epsilon(T) \leq C \Delta t.
\end{equation}

Now, fix $k$ and choose times $t_1, \dots, t_k$. We see that
\begin{align*}
    \mathbb{E}[||(\mathbf{Y}_n(t_1), \dots, \mathbf{Y}_n(t_k)) - (\mathbf{Y}(t_1), \dots, \mathbf{Y}(t_k))||^2] 
    &\leq \sum_{i=1}^k \mathbb{E}[||\mathbf{Y}_n(t_i) - \mathbf{Y}(t_i)||^2] \\
    &\leq k \sup_{0 \leq s \leq T} \mathbb{E}\left[ \left|\mathbf{Y}_n(s) - \mathbf{x}(s) \right|^2  \right] \\
    &\leq kC\Delta t \rightarrow 0
\end{align*}
as $\Delta t \rightarrow 0$. This shows $(\mathbf{Y}_n(t_1), \dots, \mathbf{Y}_n(t_k)) \xlongrightarrow{\mathcal{L}_2} (\mathbf{Y}(t_1), \dots, \mathbf{Y}(t_k))$, which implies the desired result.
\end{proof}

\section{Implementation}
\label{sec:implementation}
We use directly with no changes the models and training protocols in \cite{kingma2021variational} to parameterize our score network $\boldsymbol\epsilon(\mathbf{x}, t)$ to evaluate the log-likelihoods of our proposed diffusion models. To evaluate FID, we instead use the architecture and training procedures in \cite{karras2022elucidating}, again with no changes. All training is performed on NVIDIA RTX A6000 GPUs.

\end{document}